\documentclass{article}
\pdfoutput=1
\usepackage{makeidx}
\usepackage{bm}
\usepackage{amsmath}
\usepackage{amssymb}
\usepackage{amsthm}
\usepackage{algorithm}
\usepackage{algorithmic}
\usepackage{color}
\usepackage{xfrac}
\usepackage{dsfont}
\usepackage{graphicx}
 \usepackage[caption=false]{subfig}

\usepackage[hmargin=2cm,vmargin=2cm]{geometry}
\newcommand{\e}{\epsilon}
\newcommand{\nin}{\not\in}
\newcommand{\half}{\frac{1}{2}}

\newcommand{\set}[1]{\left\{#1\right\}}
\newcommand{\sbrk}[1]{\left[#1\right]}
\newcommand{\paren}[1]{\left(#1\right)}

\newcommand{\floor}[1]{\left\lfloor#1\right\rfloor}
\newcommand{\ceil}[1]{\left\lceil#1\right\rceil}

\newcommand{\given}[2]{\left.#1\right|#2}
\newcommand{\mE}[1]{\mathbb{E}\left[#1\right]}
\newcommand{\mP}[1]{\mathbb{P}\left[#1\right]}
\newcommand{\cP}[2]{\mathbb{P}\left[\given{#1}{#2}\right]}

\newcommand{\figref}[1]{Figure \ref{#1}}
\newcommand{\secref}[1]{Section \ref{#1}}
\newcommand{\algref}[1]{Algorithm \ref{#1}}
\newcommand{\lemref}[1]{Lemma \ref{#1}}

\newcommand{\corref}[1]{Corollary \ref{#1}}
\newcommand{\propref}[1]{Proposition \ref{#1}}

\newcommand{\algorithmicinput}{\textbf{input}}
\newcommand{\INPUT}{\item[\algorithmicinput]}

\newtheorem{thm}{Theorem}
\newtheorem{lem}{Lemma}

\newtheorem{cor}{Corollary}
\newtheorem{prop}{Proposition}
\newtheorem{definition}{Definition}

\begin{document}

\title{Concurrent bandits and cognitive radio networks}

%\author{Orly Avner\inst{1} \and Shie Mannor\inst{2}}
%\institute{Department of Electrical Engineering, Technion, Haifa\\
%\email{orlyka@tx.technion.ac.il}
%\and
%Department of Electrical Engineering, Technion - Israel Institute of Technology, Haifa, Israel}

\author{Orly Avner and Shie Mannor \\
Department of Electrical Engineering, \\
Technion - Israel Institute of Technology, Haifa, Israel}

\date{}

\maketitle

\begin{abstract}
We consider the problem of multiple users targeting the arms of a single multi-armed stochastic bandit. The motivation for this problem comes from cognitive radio networks, where selfish users need to coexist without any side communication between them, implicit cooperation or common control. Even the number of users may be unknown and can vary as users join or leave the network. We propose an algorithm that combines an $\epsilon$-greedy learning rule with a collision avoidance mechanism. We analyze its regret with respect to the system-wide optimum and show that sub-linear regret can be obtained in this setting. Experiments show dramatic improvement compared to other algorithms for this setting.
\end{abstract}

\section{Introduction}
\label{sec:intro}
In this paper we address a fundamental challenge arising in dynamic multi-user communication networks, inspired by the field of Cognitive Radio Networks (CRNs). We model a network of independent users competing over communication channels, represented by the arms of a stochastic multi-armed bandit. We begin by explaining the background, describing the general model, reviewing previous work and introducing our contribution.

\subsection{Cognitive Radio Networks}
Cognitive radio networks, introduced in \cite{Mitola1999}, refer to an emerging field in multi-user multi-media communication networks. They encompass a wide range of challenges stemming from the dynamic and stochastic nature of these networks. Users in such networks are often divided into primary and secondary users. The primary users are licensed users who enjoy precedence over secondary users in terms of access to network resources. The secondary users face the challenge of identifying and exploiting available resources. Typically, the characteristics of the primary users vary slowly, while the characteristics of secondary users tend to be dynamic. In most realistic scenarios, secondary users are unaware of each other. Thus, there is no reason to assume the existence of any cooperation or communication between them. Furthermore, they are unlikely to know even the \emph{number} of secondary users in the system.
Another dominant feature of CRNs is their distributed nature, in the sense that a central control does not exist.

The resulting problem is quite challenging: multiple users, coexisting in an environment whose characteristics are initially unknown, acting selfishly in order to achieve an individual performance criterion. We approach this problem from the point of view of a single secondary user, and introduce an algorithm which, when applied by all secondary users in the network, enjoys promising performance guarantees.

\subsection{Multi-armed bandits}
Multi-Armed Bandits (MABs) are a well-known framework in machine learning \cite{Berry1985}. They succeed in capturing the trade-off between exploration and exploitation in sequential decision problems, and have been used in the context of learning in CRNs over the last few years \cite{Jouini2010}, \cite{Avner2011,Avner2012}. Classical bandit problems comprise an agent (user) repeatedly choosing a single option (arm) from a set of options whose characteristics are initially unknown, receiving a certain reward based on every choice. The agent wishes to maximize the acquired reward, and in order to do so she must balance exploration of unknown arms and exploitation of seemingly attractive ones. Different algorithms have been proposed and proved optimal for the stochastic setting of this problem \cite{Auer2002a},\cite{Auer2010},\cite{Garivier2011}, as well as for the adversarial setting \cite{Auer2002b}.

We adopt the MAB framework in order to capture the challenge presented to a secondary user choosing between several unknown communication channels. The characteristics of the channels are assumed to be fixed, corresponding to a relatively slow evolution of primary user characteristics. The challenge we address in this paper arises from the fact that there are \emph{multiple} secondary users in the network.

\subsection{Multiple users playing a MAB}
A natural extension of the CRN-MAB framework described above considers multiple users attempting to exploit resources represented by \emph{the same} bandit. The multi-user setting leads to collisions between users, due to both exploration and exploitation; an ``attractive'' arm in terms of reward will be targeted by all users, once it has been identified as such. In real-life communication systems, collisions result in impaired performance. In our model, reward loss is the natural manifestation of collisions.

As one might expect, straightforward applications of classical bandit algorithms designed for the single-user case, e.g., KL-UCB \cite{Garivier2011}, are hardly beneficial. The reason is that in the absence of some form of a collision avoidance mechanism, all users attempt to sample the same arm after some time. We illustrate this in \secref{sec:experiments}.

We therefore face the problem of sharing a resource and learning its characteristics when users cannot communicate and are oblivious to each other's existence.

\subsection{Related work}
Recently, considerable effort has been put into finding a solution for the multi-user CRN-MAB problem.
One approach, considered in \cite{Liu2010}, is based on a Time-Division Fair Sharing (TDFS) of the best arms between all users. This policy enjoys good performance guarantees but has two significant drawbacks. First, the number of users is assumed to be fixed and known to all users, and second, the implementation of a TDFS mechanism requires pre-agreement among users to coordinate a time division schedule.
Another work that deals with multi-user access to resources, but does not incorporate the MAB setting, is presented in \cite{Leith2012}. The users reach an orthogonal configuration without pre-agreement or communication, using multiplicative updates of channel sampling probabilities based on collision information. However, this approach does not handle the learning aspect of the problem and disregards differences in the performance of different channels. Thus, it cannot be applied to our problem.
The authors in \cite{Kalathil2012b} consider a form of the CRN-MAB problem in which channels appear different to different users, and propose an algorithm which enjoys good performance guarantees. However, their algorithm includes a negotiation phase, based on the Bertsekas auction algorithm, during which the users communicate in order to reach an orthogonal configuration.

The work closest in spirit to ours is \cite{Anandkumar2011}. The authors propose different algorithms for solving the CRN-MAB problem, attempting to lift assumptions of cooperation and communication as they go along. Their main contribution is expressed in an algorithm which is coordination and communication free, but relies on exact knowledge of the number of users in the network. In order to resolve this issue, an algorithm which is based on \emph{estimating} the number of users is proposed. Performance guarantees for this algorithm are rather vague, and it does not address the scenario of a time-varying number of users.

A different approach to resource allocation with multiple noncooperative users involves game theoretic concepts \cite{Nie2006,Niyato2008}. In our work we focus on cognitive, rather than strategic, users. Yet another perspective includes work on CRNs with multiple secondary users, where the emphasis is placed on collision avoidance and sensing. References such as \cite{Choe2009} and \cite{Li2012} propose ALOHA based algorithms, achieving favorable results. However, these works do not consider the learning problem we are facing, and assume all channels to be known and identical.

\subsection{Contribution}
The main contribution of our paper is suggesting an algorithm for the multi-user CRN-MAB problem, which guarantees convergence to an optimal configuration when employed by all users.
Our algorithm adheres to the strict demands imposed by the CRN environment: no communication, cooperation or coordination (control) between users, and strictly local knowledge - even the number of users is unknown to the algorithm.

Also, to the best of our knowledge, ours is the only algorithm that handles a dynamic number of users in the network successfully.

The remainder of this paper is structured as follows. \secref{sec:framework} includes a detailed description of the framework and problem formulation. \secref{sec:fixed_users} presents our algorithm along with its theoretical analysis, while \secref{sec:dynamic_users} discusses the setup of a dynamic number of users. \secref{sec:experiments} displays experimental results and \secref{sec:conclusion} concludes our work. The proofs of our results are provided in the supplementary material.

\section{Framework}
\label{sec:framework}

Our framework consists of two components: the environment and the users.
The environment is a communication system that consists of $K$ channels with different, initially unknown, reward characteristics. We model these channels as the arms of a stochastic Multi-Armed Bandit (MAB). We denote the expected values of the reward distributions by $\bm{\mu} = \paren{\mu_{1},\mu_{2},\ldots,\mu_{K}}$, and assume that channel characteristics are fixed. Rewards are assumed to be bounded in the interval $\sbrk{0,1}$.

The users are a group of non-cooperative, selfish agents. They have no means of communicating with each other and they are not subject to any form of central control. Unlike some of the previous work on this problem, we assume they have no knowledge of the number of users. In \secref{sec:fixed_users} we assume the number of users is fixed and equal to $N$, and in \secref{sec:dynamic_users} we relax this assumption; in both cases we assume $K \geq N$. Scenarios in which $K < N$ correspond to over-crowded networks and should be dealt with separately. The fact that the users share the communication network is modeled by their playing \emph{the same} MAB. Two users or more attempting to sample the same arm at the same time will encounter a collision, resulting in a zero reward for all of them in that round. A user sampling an arm $k$ alone at a certain time $t$ receives a reward $r\paren{t}$, drawn i.i.d from the distribution of arm $k$.

We would like to devise a policy that, when applied by all users, results in convergence to the system-optimal solution. A common performance measure in bandit problems is the expected regret, whose definition for the case of a single user is
\begin{align*}
\mE{R\paren{t}} \triangleq \mu_{k^*} t-\sum_{\tau=1}^t\mE{r\paren{\tau}},
\end{align*}
where $\mu_{k^*} = \max_{k\in\set{1,\ldots,K}}\mu_k$ is the expected reward of the optimal arm.

Naturally, in the multi-user scenario not all users can be allowed to select the optimal arm. Therefore, the number of users defines a \emph{set} of optimal arms, namely the $N$ best arms, which we denote by $K^*$. Thus, the appropriate expected regret definition is
\begin{align*}
\mE{R\paren{t}} \triangleq t\sum_{k \in K^*}\mu_{k} -\sum_{n = 1}^N\sum_{\tau=1}^t\mE{r_n\paren{\tau}},
\end{align*}
where $r_n\paren{\tau}$ is the reward user $n$ acquired at time $\tau$. We note that this definition corresponds to the expected loss due to a suboptimal sampling policy.

The socially optimal solution, which minimizes the expected regret for all users as a group, is for each to sample a different arm in $K^*$.
Adopting such a system-wide approach makes the most sense from an engineering point of view, since it maximizes network utilization without discriminating between users.
%We adopt such a system-wide approach as it makes the most sense when engineering algorithms for a fair, impartial network.

\section{Fixed number of users}
\label{sec:fixed_users}

In this section we introduce the policy applied by each of the users, described in \algref{alg:alg1}. Our policy is based on several principles:
\begin{enumerate}
  \item Assuming an arm that experiences a collision is an ``attractive'' arm in terms of expected reward, we would like one of the colliding users to continue sampling it.
  \item Since all users need to learn the characteristics of all arms, we would like to ensure that an arm is not sampled by a single user exclusively.
  \item To avoid frequent collisions on optimal arms, we need users to back off of arms on which they have experienced collisions.
  \item To avoid interfering with on-going transmissions in the steady state, we would like to prevent exploring users from ``throwing off'' exploiting users.
\end{enumerate}

\subsection{The MEGA algorithm}
The Multi-user $\e$-Greedy collision Avoiding (MEGA) algorithm is based on the $\e$-greedy algorithm introduced in \cite{Auer2002a}, augmented by a collision avoidance mechanism that is inspired by the classical ALOHA protocol.
%\begin{algorithm}[H]
\begin{algorithm}[h]
   \caption{Multi-user $\e$-Greedy collision Avoiding (MEGA) algorithm}
   \label{alg:alg1}
\begin{algorithmic}
    \INPUT Parameters $c$, $d$, $p_0$, $\alpha$ and $\beta$
    \STATE \textbf{init} $p = p_0$, $t = 1$, $\eta\paren{0} = 0$, $a\paren{0} \sim U\paren{\set{1,\ldots,K}}$, $t_{\text{next},k} = 1 \;\forall k$
    \STATE \textbf{note:} $\eta\paren{t}$ is a collision indicator
    \LOOP
    \IF{$\eta\paren{t-1} == 1$}
        \STATE With probability $p$ persist: $a\paren{t} = a\paren{t-1}$
            %\STATE
        \STATE With probability $1-p$ give up:
            \STATE $\quad$ Mark arm as taken until time $t_{\text{next},k}$, where $t_{\text{next},k}\sim U\paren{\sbrk{t,t + t^{\beta}}}$
            \STATE $\quad$ $p \leftarrow p_0$
    \ELSE
        \STATE $p \leftarrow p\cdot\alpha + \paren{1-\alpha}$
        \STATE Update $\hat{\mu}_{a\paren{t-1}}$
    \ENDIF
    \STATE Identify available arms: $A = \set{k:t_{\text{next},k} \leq t}$
    \IF{$A=\emptyset$}
        \STATE Refrain from transmitting in this round
    \ENDIF
    \STATE With probability $\e_t = \min\set{1,\frac{cK^2}{d^2\paren{K-1}t}}$ explore: $a\paren{t}\sim U\paren{A}$
        %\STATE $\quad$ $a\paren{t}\sim U\paren{A}$
    \STATE With probability $1-\e_t$ exploit:
        % \STATE $\quad$ $a\paren{t} = \argmax_{k \in A}\hat{\mu}_k$
    \IF{$a\paren{t} \neq a\paren{t-1}$}
        \STATE $p \leftarrow p_0$
    \ENDIF
    \STATE Sample arm $a\paren{t}$ and observe $r\paren{t}, \eta\paren{t}$
    \ENDLOOP
    \STATE \textbf{note:} $\hat{\mu}_k$ is the empirical mean of an arm's reward
\end{algorithmic}
\end{algorithm}

Learning is achieved by balancing exploration and exploitation through a time-dependant exploration probability. The collision avoidance mechanism is based on the idea that users sampling an arm have a persistence probability that controls their ``determination'' once a collision occurs. This probability, denoted in \algref{alg:alg1} by $p$, depends on the number of their consecutive successful sample attempts. Its initial value is $p_0$, and it is incremented with each successful sample. Once a collision event begins, the persistence probability remains fixed until it ends.

A collision event ends when all users but one have ``given up'' and stopped sampling the arm under dispute.
Upon giving up, each user resets her persistence and draws a random interval of time during which she must refrain from sampling the arm under dispute. The length of these intervals increases over time in order to ensure sub-linear regret.

The parameters in \algref{alg:alg1} should be chosen so that $p_0$, $\alpha$ and $\beta$ are all in the interval $\paren{0,1}$. In the original $\e$-greedy algorithm, the parameter $d$ is set to be $\mu_{k^*} - \mu_{k_2}$, where $\mu_{k_2}$ is the expected reward of the second-best arm. In our case, learning the $N$ best arms requires that $d$ be modified and set to $\mu_{k_{N-1}} - \mu_{k_N}$, where $\mu_{k_i}$ is the expected reward of the $i$-best arm. However, since the expected rewards of the arms are unknown in practice and we assume the number of users to be unknown, we use a fixed value for $d$ in our experiments. For details see \secref{sec:experiments}.

The exploration probability, $\e_t$, is modified compared to the original $\e$-greedy algorithm \cite{Auer2002a}, in order to account for the decreased efficiency of samples, caused by collisions. For our algorithm we use $\e_t = \min\set{1,\frac{cK^2}{d^2\paren{K-1}t}}$. Also, the empirical mean which determines the ranking of the arms is calculated based on the number of successful samples of each arm.

\subsection{Analysis of the MEGA algorithm}

We now turn to a theoretical analysis of the MEGA algorithm. Our analysis shows that when all users apply MEGA, the expected regret grows at a sub-linear rate, i.e., MEGA is a no-regret algorithm.

The regret obtained by users employing the MEGA algorithm consists of three components.
The first component is the loss of reward due to collisions: in a certain round $t$, all colliding users receive zero reward. We denote the expected reward loss due to collisions by $\mE{R^C\paren{t}}$.
The second and third components reflect the loss of reward due to sampling of suboptimal arms, i.e., arms $k \nin K^*$. Once the users have learned the ranking of the different arms, suboptimal sampling is caused either by random exploration, dictated by the $\e$-greedy algorithm, or due to the fact that all arms in $K^*$ are marked unavailable by a user at a certain time. We denote the expected reward loss due to these issues by $\mE{R^E\paren{t}}$ and $\mE{R^A\paren{t}}$, respectively.

We begin by showing that all users succeed in learning the correct ranking of the $N$-best arms in finite time in \lemref{lem:T}. This result will serve as a base for the bounds of the different regret components.

\begin{definition}\label{def:ranking}
  An $\e$-correct ranking of $M$ arms is a sorted $M$-vector of empirical mean rewards of arms (i.e., $i < j \iff \hat{\mu}_i \leq \hat{\mu}_j$), such that
  \begin{align*}
  \hat{\mu}_i \leq \hat{\mu}_j \iff {\mu}_i + \e \leq {\mu}_j \;\;\forall i,j\in\set{1,\ldots,M}, i\neq j.
  \end{align*}
\end{definition}

\begin{lem}\label{lem:T}
  For a system of $K$ arms and $N$ users, $N\leq K$, in which all users employ MEGA, there exists a finite time
  $T = 2\frac{4K^{N}N}{\e^2\prod_{i=1}^{N-1}\paren{K-i}}\log\paren{\frac{2K}{\delta}}$ such that $\forall t>T$, all users have learned an $\e$-correct ranking of the $N$-best arms with a probability of at least $1-\delta$.
\end{lem}

\begin{proof}
We prove the existence of a finite $T$ by combining the sample complexity of stochastic MABs with the characteristics of MEGA.

First, we note that as long as $\e_t = 1$, if the availability mechanism is disabled, each of the users performs uniform sampling on average.
We therefore examine a slightly modified version of MEGA for the sake of this theoretical analysis.

Based on \cite{Even2002}, a na\"{\i}ve algorithm that samples each arm $\ell_K = \frac{4}{\e^2}\log\paren{\frac{2K}{\delta}}$ times, identifies an $\e$-best arm with probability of at least $1-\delta$. A loose bound on the number of samples needed in order to produce a correct \emph{ranking} of the N-best arms of a K-armed bandit, is obtained by applying an iterative procedure: sample each of the $K$ arms $\ell_K$ times and select the best arm; then sample each of the remaining $K-1$ arms $\ell_{K-1}$ times and select the best arm; repeat the procedure $N$ times. Such an approach requires no more than $S = \frac{4N}{\e^2}\log\paren{\frac{2K}{\delta}}$ samples of each arm, for each user.

The collision probability of $N$ users uniformly sampling $K$ channels (in the absence of an availability mechanism) is given by the solution of the well-known ``birthday problem'' \cite{Mckinney1966}:
\begin{align*}
  \mP{C} = 1 - \prod_{d = 1}^{N-1}\paren{1 - \frac{d}{K}}.
\end{align*}
As a result of the collisions, the number of samples which are ``effective'' in terms of learning arm statistics is reduced. For a certain arm $k$, sampled by a user $n$, the expected number of successful samples up till time $t$ is given by
\begin{align*}
  \mE{s_{k,n}\paren{t}} = \paren{1-\mP{C}}\frac{t}{K} = \frac{t}{K}\prod_{i = 1}^{N-1}\paren{1 - \frac{i}{K}}.
\end{align*}
In order to ensure an adequate number of samples we need to choose a certain $T'$ for which $\mE{s_{k,n}\paren{T'}} = S$:
\begin{align*}
  \frac{T'}{K}\prod_{i = 1}^{N-1}\paren{1 - \frac{i}{K}} = S,
\end{align*}
meaning that
\begin{align*}
  T' = \frac{4K^{N}N}{\e^2\prod_{i=1}^{N-1}\paren{K-i}}\log\paren{\frac{2K}{\delta}}.
\end{align*}
Since the users' sampling is random, it is only uniform on average. By choosing $T = 2T'$, we ensure that the number of samples is sufficient with high probability:
\begin{align*}
  \mP{\mE{s_{k,n}\paren{2T'}} - s_{k,n}\paren{2T'} > S} \leq e^{\sfrac{-S^2}{T'}} \leq \paren{\frac{2K}{\delta}}^{-\frac{1}{\e^2}\frac{4N}{K}\paren{\frac{K-N}{K}}^{N-1}},
\end{align*}
which is due to Hoeffding's inequality.

We note that \lemref{lem:T} holds for a choice of the parameter $c$ which ensures that $\e_t = 1 \quad \forall t<T$:
\begin{align*}
  c = \frac{d^2\paren{K-1}T}{K^2}.
\end{align*}
$\hfill\square$
\end{proof}

Based on \lemref{lem:T} we proceed with the analysis of MEGA, incorporating the fact that for all $t>T$, all users know the correct ranking of the $N$ best arms.

Back to our regret analysis - since the reward is bounded in $\sbrk{0,1}$, the expected regret is also bounded:
\begin{align*}
  \mE{R\paren{t}} \leq \mE{R^C\paren{t}} + \mE{R^E\paren{t}} + \mE{R^A\paren{t}}.
\end{align*}

%% Collisions

We begin by addressing the expected regret due to collisions, denoted by $\mE{R^C\paren{t}}$. The bound on collision regret is derived from a bound on the total number of collisions between two users on a single channel up till $t$, $C_p\paren{t}$, whose expected value is bounded in \lemref{lem:collisions}.

\begin{lem}
\label{lem:collisions}
The expected number of collisions between two users on a single channel up till time $t$ is bounded:
\begin{align}\label{eq:CtBound}
  \mE{C_p\paren{t}} \leq \frac{2L_{\text{up}}}{\sqrt{L_{\text{low}}}}t^{1-\beta/2},
\end{align}
where the constants are defined in the proof and $\beta$ is a parameter of MEGA.
\end{lem}

Once we have a bound for the pairwise, per-arm, number of collisions, we can bound the mean number of collisions for all users.
\begin{cor}
\label{cor:collisions}
The expected number of collisions between all users over all channels up till time $t$ is bounded:
\begin{align*}
  \mE{C\paren{t}} \leq \half N \paren{N-1} K \mE{C_p\paren{t}} \leq N^2 K \frac{L_{\text{up}}}{\sqrt{L_{\text{low}}}}t^{1-\beta/2}.
\end{align*}
Since the reward is bounded in $\sbrk{0,1}$, the expected regret acquired as a result of collisions up till time $t$ is bounded by the same value:
\begin{align*}
  \mE{R^C\paren{t}} \leq N \paren{N-1} K \mE{C_p\paren{t}} \leq C_1 N^2 K t^{1-\beta/2},
\end{align*}
where $C_1 = \frac{L_{\text{up}}}{\sqrt{L_{\text{low}}}}$.
\end{cor}
\corref{cor:collisions} follows from \lemref{lem:collisions}, since each pair of users can collide on each arm, before the dictated ``quiet'' period, and so we obtain a bound on the expected regret accumulated due to collisions.

%% Availability

Next, we examine the expected regret caused by the unavailability of arms in $K^*$, denoted by $\mE{R^A\paren{t}}$.
The availability mechanism contributes to a user's regret if it marks all arms in $K^*$ as taken, causing the user to choose an arm $k\nin K^*$ until one of the arms in $K^*$ becomes available once again.

We compute an upper bound on the regret by analyzing the regret due to unavailability when the number of users is $N=2$. When there are more users, the regret is bounded by the worst case, in which all of them declare an optimal arm unavailable at the same time:
\begin{align*}
  \mE{R^A\paren{t}} \leq N\mE{R^A_2\paren{t}},
\end{align*}
where $\mE{R^A_2\paren{t}}$ is the availability regret accumulated up till time $t$ in the two user scenario for a single channel.

\begin{lem}\label{lem:avail}
The expected regret accumulated due to unavailability of optimal arms in the interval $\sbrk{T,t}$ is bounded:
\begin{align*}
  \mE{R^A_2\paren{t}} \leq C_3 t^\beta,
\end{align*}
where the constant $C_3$ is defined in the proof and $\beta$ is a parameter of MEGA.
\end{lem}

\corref{cor:availability} follows from \lemref{lem:avail}:

\begin{cor}
\label{cor:availability}
  The expected regret contributed by the availability-detection mechanism up till time $t$ is bounded:
  \begin{align*}
    \mE{R^A\paren{t}} \leq NKT + NK C_3 t^\beta.
  \end{align*}
\end{cor}

Our next goal is to bound the regret due to exploration, which is dictated by the $\e$-greedy approach adopted in MEGA.

\begin{lem}\label{lem:explore}
  The expected regret accumulated by all users employing the MEGA algorithm due to random exploration up till time $t$ is bounded $\forall t > m$:
  \begin{align*}
    \mE{R^E\paren{t}} \leq Nm + \frac{cK^2N}{d^2{K-1}}\log t,
  \end{align*}
  where $c,d$ are parameters of the MEGA algorithm and $m = \frac{cK^2}{d^2\paren{K-1}}$.
\end{lem}

Based on the lemmas and corollaries above, we have the following regret bound for the MEGA algorithm:
\begin{thm}
  Assume a network consisting of $N$ users playing a single $K$-armed stochastic bandit, $N\leq K$. If all users employ the policy of MEGA, the system-wide regret is bounded for all $t>\max\paren{m,T}$ as follows:
  \begin{align*}
    \mE{R\paren{t}} \leq &\;C_1 N^2 K t^{1-\beta/2} + NKT + NK C_3 t^\beta + Nm + \frac{cK^2N}{d^2{K-1}}\log t \\
                =& \;O\paren{t^{1-\beta/2} + \log t + t^\beta}.
  \end{align*}
\end{thm}
The dominant term in the regret bound above depends on the value of $\beta$. For $\beta > 2/3$, the term $t^\beta$ dominates the bound, while for smaller values the dominant term is $t^{1 - \beta/2}$. This tradeoff is intuitive - large values of $\beta$ correspond to longer ``quiet'' intervals, reducing the regret contributed by collisions. However, such long intervals also result in longer unavailability periods, increasing the regret contributed by the availability mechanism. Optimizing over $\beta$ yields $\beta = 2/3$, and so the corresponding regret bound is
\begin{align*}
  R\paren{t} \leq O\paren{t^{\frac{2}{3}}}.
\end{align*}

The regret bounds for the algorithms proposed in \cite{Anandkumar2011} and \cite{Kalathil2012b} are $O\paren{\log t}$ and $O\paren{\log^2 t}$, respectively. It is worth noting that the constants in the bound provided in \cite{Anandkumar2011} are very large, as they involve a binomial coefficient which depends on the numbers of users and channels. Also, the assumptions our algorithm makes are much more strict. Reference \cite{Anandkumar2011} requires knowing the number of users, and \cite{Kalathil2012b} requires ongoing communication between users, through the Bertsekas auction algorithm. Reference \cite{Anandkumar2011} does propose an algorithm which estimates the number of users, but its regret bound is not logarithmic, and rather vague. In addition, the empirical results our algorithm provides are considerably better (see \secref{sec:experiments}).

\section{Dynamic number of users}
\label{sec:dynamic_users}
So far, we have focused on several traits of the MEGA algorithm: it does not require communication, cooperation or coordination among users, and it does not assume prior knowledge of the number of users. Simply put, the user operates as though she were the only user in the system, and the algorithm ensures this approach does not result in the users' interfering with each other once the system reaches a ``steady state''.

However, communication networks like the ones we wish to model often evolve over time - users come and go, affecting the performance of other users by their mere presence or absence. As mentioned in \secref{sec:intro}, the algorithms proposed in \cite{Anandkumar2011} attempt to address scenarios similar to ours. However, they rely on either \emph{knowing} or \emph{estimating} the number of users. Thus, a varying number of users is beyond their scope.

It is evident from the experiments in \secref{sec:experiments} that the MEGA algorithm is applicable not only to a fixed number of users, but also in the case that the number of active users in the network varies over time. To the best of our knowledge, this is the \emph{only} algorithm that is able to handle such a setup.

We defer a thorough analysis of the dynamic scenario to our future work. However, a simple performance guarantee can be obtained for the event in which a user leaves the network. Let us begin by defining the regret. Let $N\paren{t}$ denote the number of users in the network at time $t$. Accordingly, $K^*\paren{t}$ is the set of $N\paren{t}$-best arms. The series $t_1,t_2,t_3\ldots$ denotes change events in the number of users - arrival or departure. Time intervals during which the number of users is fixed are denoted by $T_i\triangleq\sbrk{t_{i-1},t_i-1}$, with $t_0 \triangleq 0$. Following the definition of the regret introduced in \secref{sec:framework}, the regret for the dynamic scenario is
\begin{align*}
  R\paren{t} \triangleq \sum_{T_i}\sum_{k\in K^*\paren{t_{i-1}}}\mu_k - \sum_{\tau\in T_i}\sum_{n\in N\paren{t_{i-1}}}\mE{r_n\paren{\tau}},
\end{align*}
where we allow a slight abuse of notation for the sake of readability.

Let us assume that a user $n$ leaves the network at some time $t$, and that the number of users without him is $N\paren{t}$. We also assume that the users had reached a steady state before this departure, i.e., the optimal configuration was being sampled with high probability. Unless user $n$ was sampling the $N\paren{t}+1$-best arm, regret will start building up at this point. Based on \propref{prop:dynamic1}, which follows directly from the definition of the MEGA algorithm, we bound the regret accumulated until the system ``settles down'' in the new optimal configuration, in \propref{prop:dynamic2}.

\begin{figure}[ht]
\vskip 0.2in
\begin{center}
\centerline{\includegraphics[width=0.5\columnwidth]{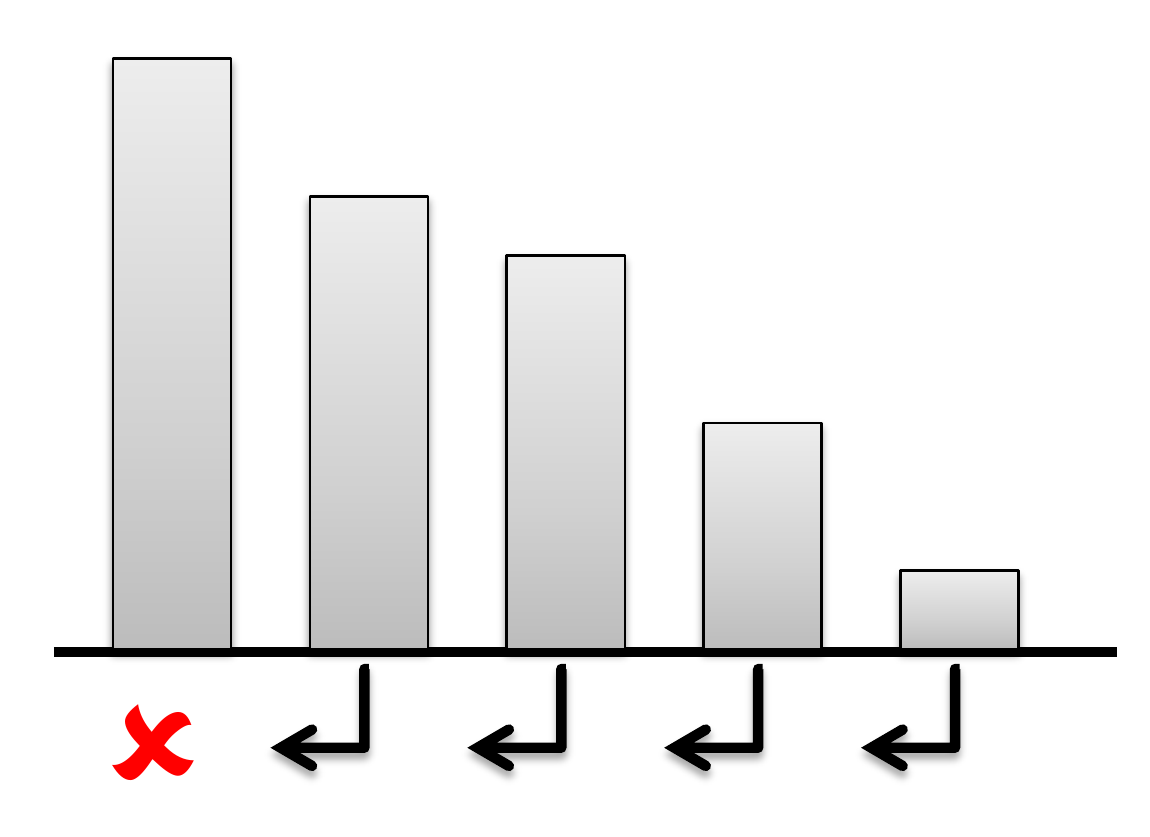}}
\caption{Worst case occupation of new optimal configuration after a user has left the network.}
\label{fig:dynamic}
\end{center}
\vskip -0.2in
\end{figure}

\begin{prop}
\label{prop:dynamic1}
  Let $K^*_{n}$ denote the set of $n$-optimal arms.
  For an arm such that $k\in K^*_{n}$ and also $k\in K^*_{n-1}$, if a user occupying $k$ becomes inactive at time $t$, $k$ will return to the set of regularly sampled arms within a period of no more than $t^\beta$, with a probability greater than $1-\e_t$.
\end{prop}

\begin{prop}
\label{prop:dynamic2}
The regret accumulated in the period between user $n$'s departure at time $t$ and the new optimal configuration's being reached is bounded by
\begin{align*}
R\paren{t} \leq \frac{2^{\paren{\beta+1}\paren{N\paren{t}-1}}-1}{2^{\beta+1}-1}t^\beta = O\paren{t^\beta}.
\end{align*}
\end{prop}

\propref{prop:dynamic2} follows from \propref{prop:dynamic1} and from the worst case analysis described in \figref{fig:dynamic}: if the freed arm was the best one, and the user sampling the second-best arm re-occupied it, then the second-best arm would be left to be occupied, etc.. In the worst case, the time intervals before users ``upgrade'' their arms are back-to-back, creating a series of the form $t^\beta, \paren{t+t^\beta}^\beta,\ldots$. For detailed proofs of these propositions see the supplementary material.

\propref{prop:dynamic2} shows that for a sufficiently low departure rate, regret remains sub-linear even in the scenario of a dynamic number of users. Clearly, frequent changes in the number of users will result in linear regret for practically any distributed algorithm, including ours.
\section{Experiments}
\label{sec:experiments}
Our experiments simulate a cognitive radio network with $K$ channels and $N$ users. The existence of primary users is manifested in the differences in expected reward yielded by the channels (i.e., a channel taken by a primary user will yield small reward). Over the course of the experiment, the secondary users learn the characteristics of the communication channels and settle into an orthogonal, reward-optimal transmission configuration.

The first experiments concern a fixed number of users, $N$, and channels, $K$. We assume channel rewards to be Bernoulli random variables with expected values $\bm{\mu} = \paren{\mu_{1},\mu_{2},\ldots,\mu_{K}}$. The initial knowledge users have is only of the number of channels.

Once again, we stress that our users \emph{do not} communicate among themselves, nor do they receive external control signals. Their only feedback is the instantaneous reward and an instantaneous collision indicator.

We begin by showing that straightforward application of classic bandit algorithms does not suffice in this case. \figref{fig:regretUCB} and \figref{fig:collisionsUCB} present simulation results for a basic scenario in which $N=K=2$. Even in this rather simple case, the KL-UCB and $\e$-greedy algorithms fail to converge to an orthogonal configuration, and the number of collisions between users grows linearly with time. The experiment was repeated 50 times.

\begin{figure}[ht]
\label{fig:basicUCB}
\begin{center}
  \subfloat[Average regret over time]{\includegraphics[width=0.5\textwidth]{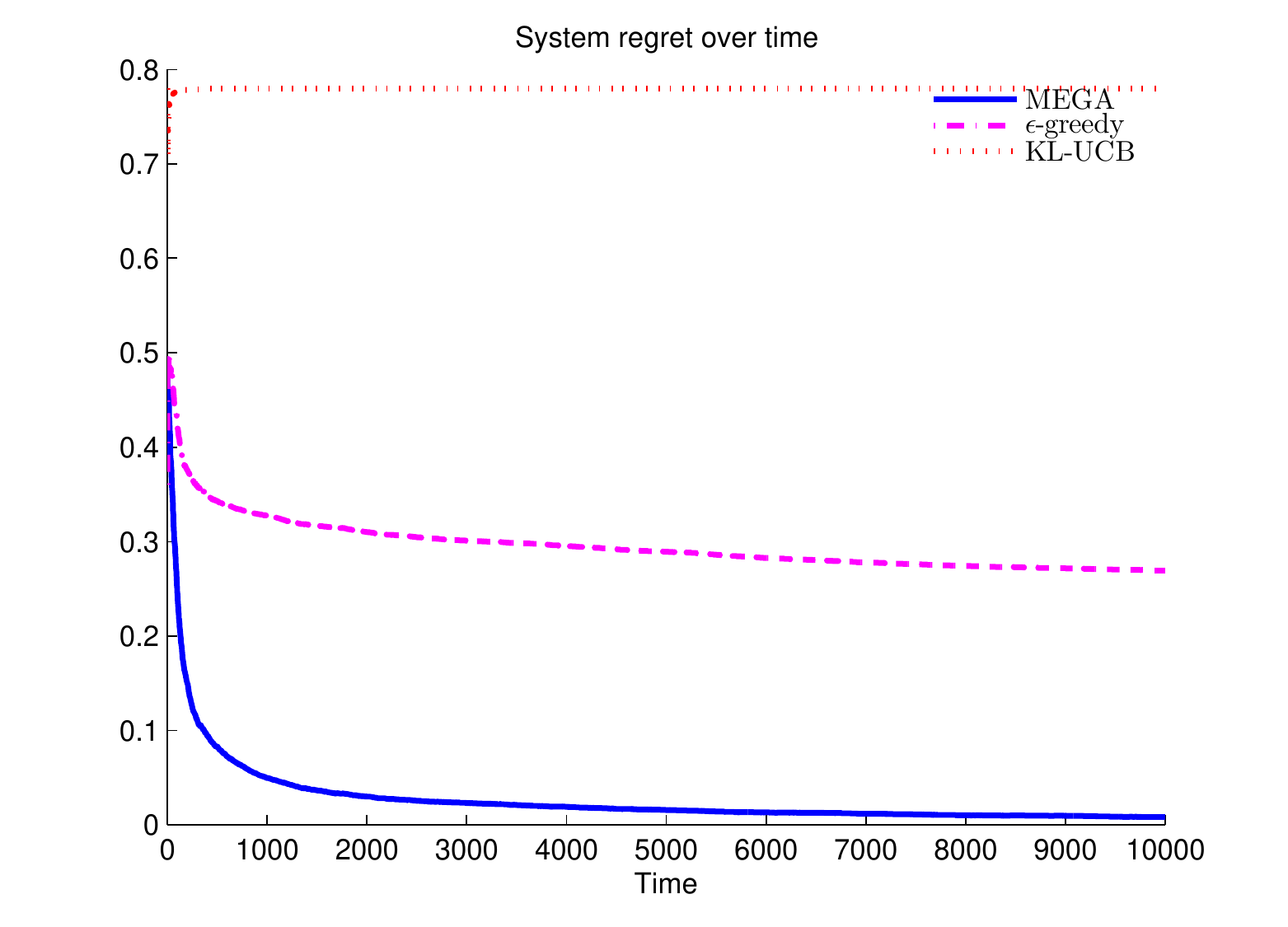}\label{fig:regretUCB}}
  \subfloat[Collisions over time]{\includegraphics[width=0.5\textwidth]{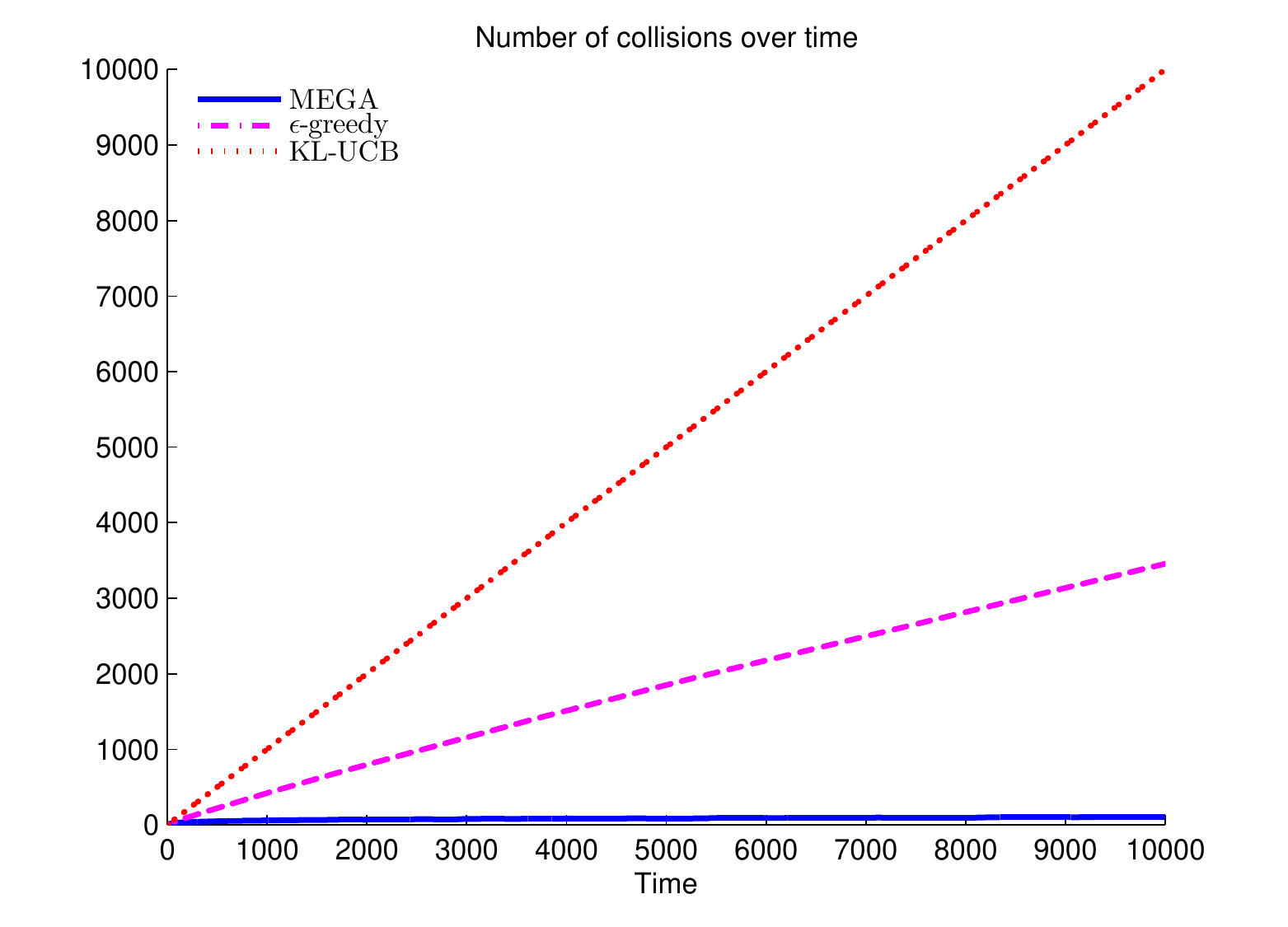}\label{fig:collisionsUCB}}
  \caption{KL-UCB, $\e$-greedy and MEGA performance in basic scenario}
\end{center}
\end{figure}

Having demonstrated the need for an algorithm tailored to our problem, we compare the performance of our algorithm, MEGA, with the $\rho^{\text{RAND}}$ algorithm, proposed in \cite{Anandkumar2011}. \figref{fig:regret_fixed} displays the average regret over time and \figref{fig:collisions_fixed} displays the cumulative number of collisions over time, averaged over all users. An important note is that in our experiments we provide $\rho^{\text{RAND}}$ with the exact number of users, as it requires. The MEGA algorithm does not require this input. We did not implement the algorithm $\rho^{\text{EST}}$ \cite{Anandkumar2011}, as its pseudo-code was rather obscure.

The set of parameters used for MEGA was determined by cross validation: $c = 0.1, p_0 = 0.6, \alpha = 0.5, \beta = 0.8$. The value of $d$ as dictated by the $\e$-greedy algorithm should be $d\leq\Delta = \mu_{N-\text{best}} - \mu_{\paren{N-1}-\text{best}}$. Calculating this value requires prior knowledge of both the number of users and the channels' expected rewards. In order to avoid this issue, we set $d = 0.05$ and avoided distributions for which this condition does not hold. The algorithm $\rho^{\text{RAND}}$ is parameter-free, as it is a modification of the UCB1 algorithm \cite{Auer2002a}.

\begin{figure}\label{fig:fixed}
\begin{center}
  \subfloat[Average regret over time]{\includegraphics[width=0.5\textwidth]{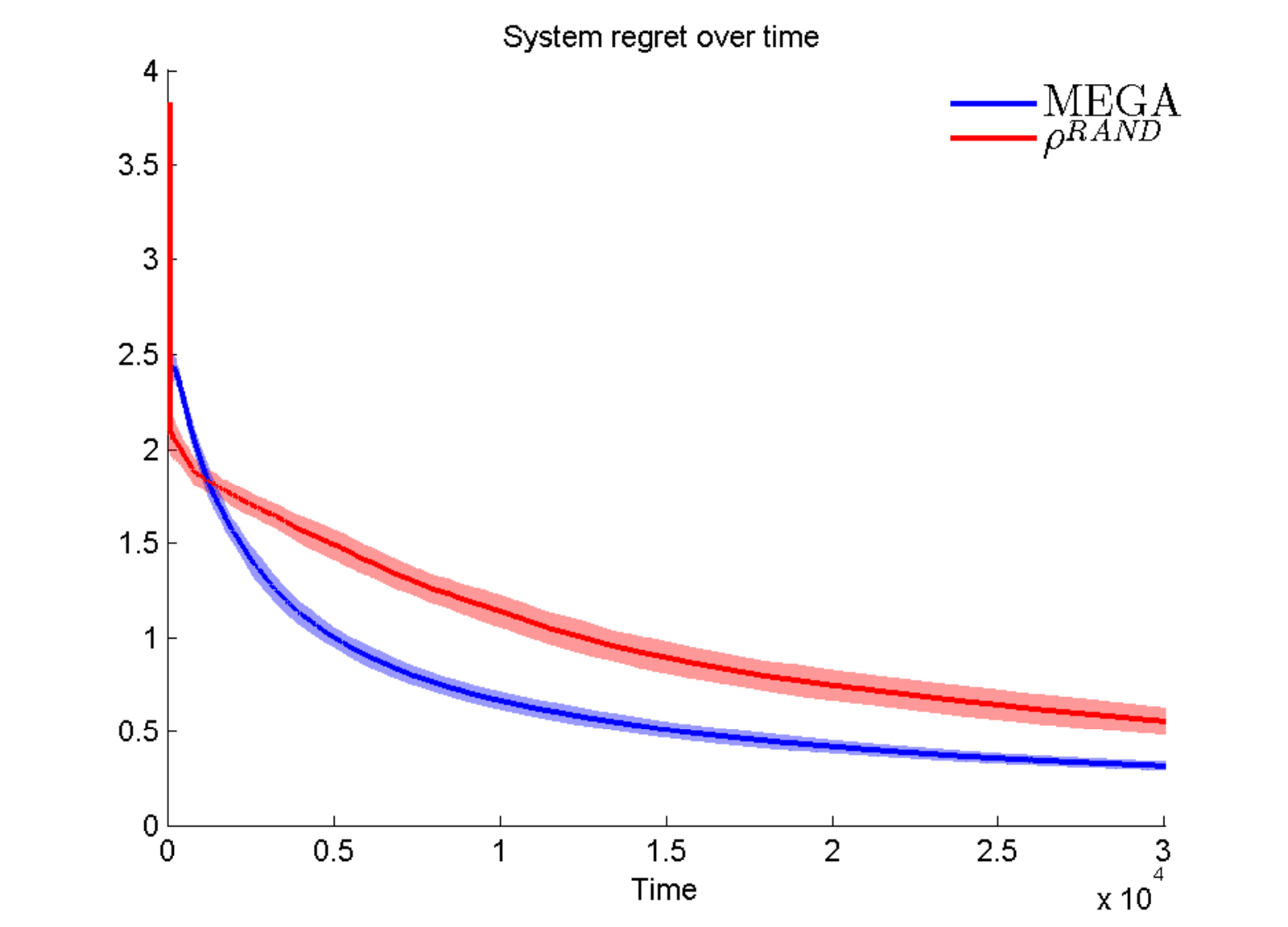}\label{fig:regret_fixed}}
  \subfloat[Collisions over time]{\includegraphics[width=0.5\textwidth]{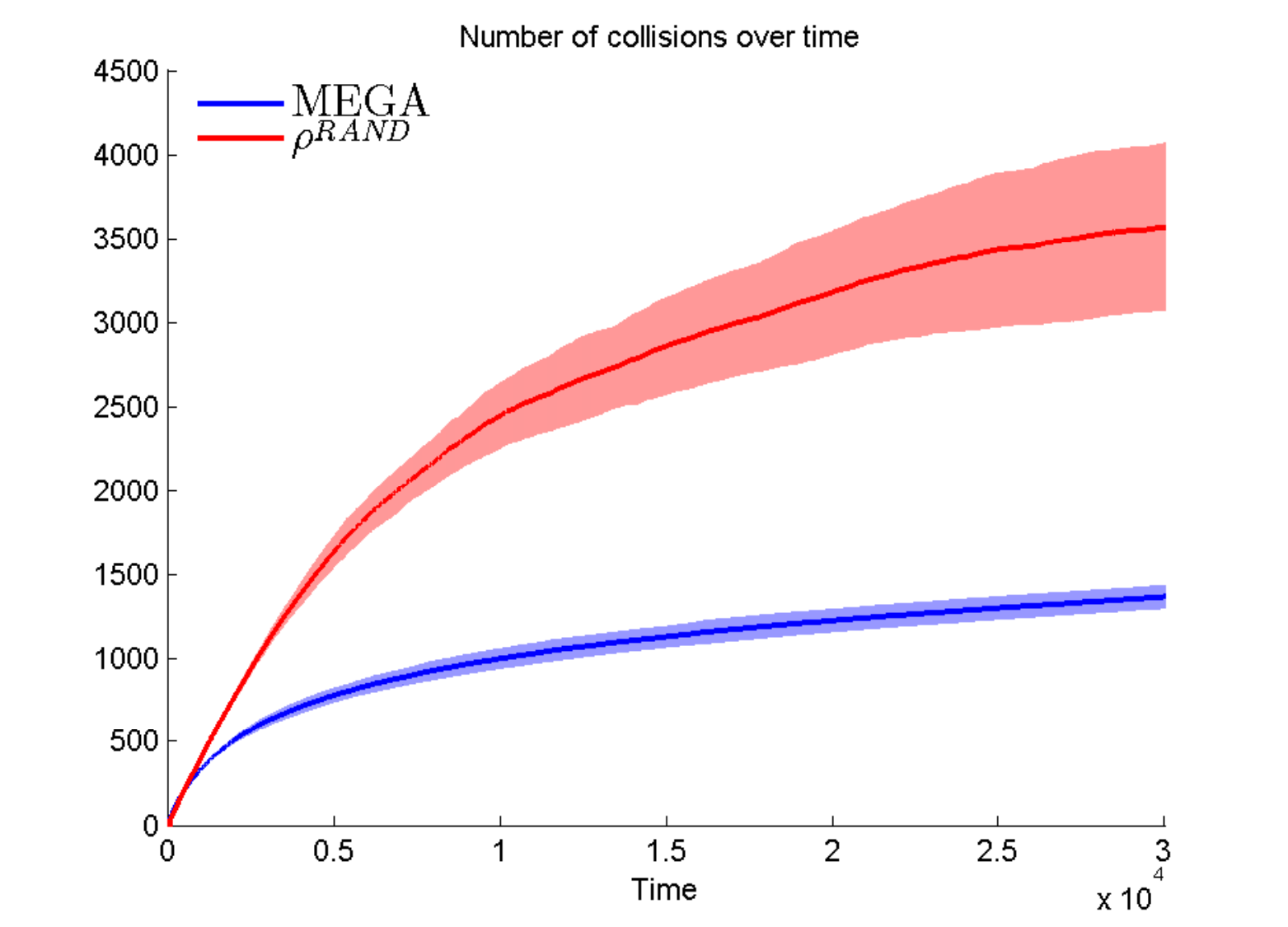}\label{fig:collisions_fixed}}
  \caption{Performance of MEGA compared to $\rho^{\text{RAND}}$. The shaded area around the line plots represents result variance over 50 repetitions. The experiment was run with $N=6$ users and $K=9$ channels, and the number of collisions was averaged over all users.}
\end{center}
\end{figure}

The results in \figref{fig:regret_fixed} and \figref{fig:collisions_fixed} present a scenario in which $N<K$. In the more challenging scenario of $N=K$ our algorithm's advantage is even more pronounced, as is evident from \figref{fig:regret_fixedKN} and \figref{fig:collisions_fixedKN}. Here, $\rho^{\text{RAND}}$ actually fails to converge to the optimal configuration, yielding constant average regret.

\begin{figure}\label{fig:KN}
\begin{center}
  \subfloat[Average regret over time]{\includegraphics[width=0.5\textwidth]{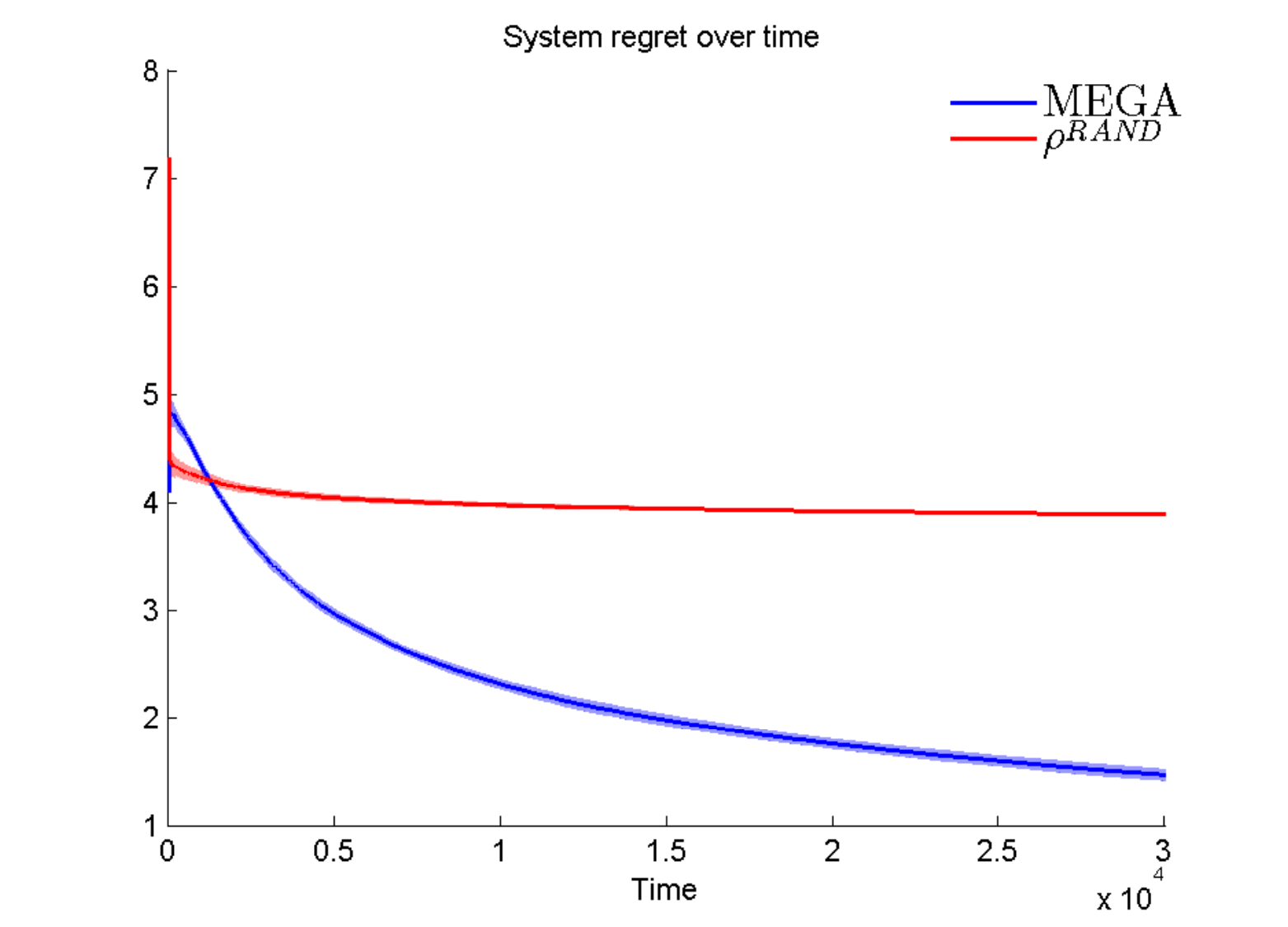}\label{fig:regret_fixedKN}}
  \subfloat[Collisions over time]{\includegraphics[width=0.5\textwidth]{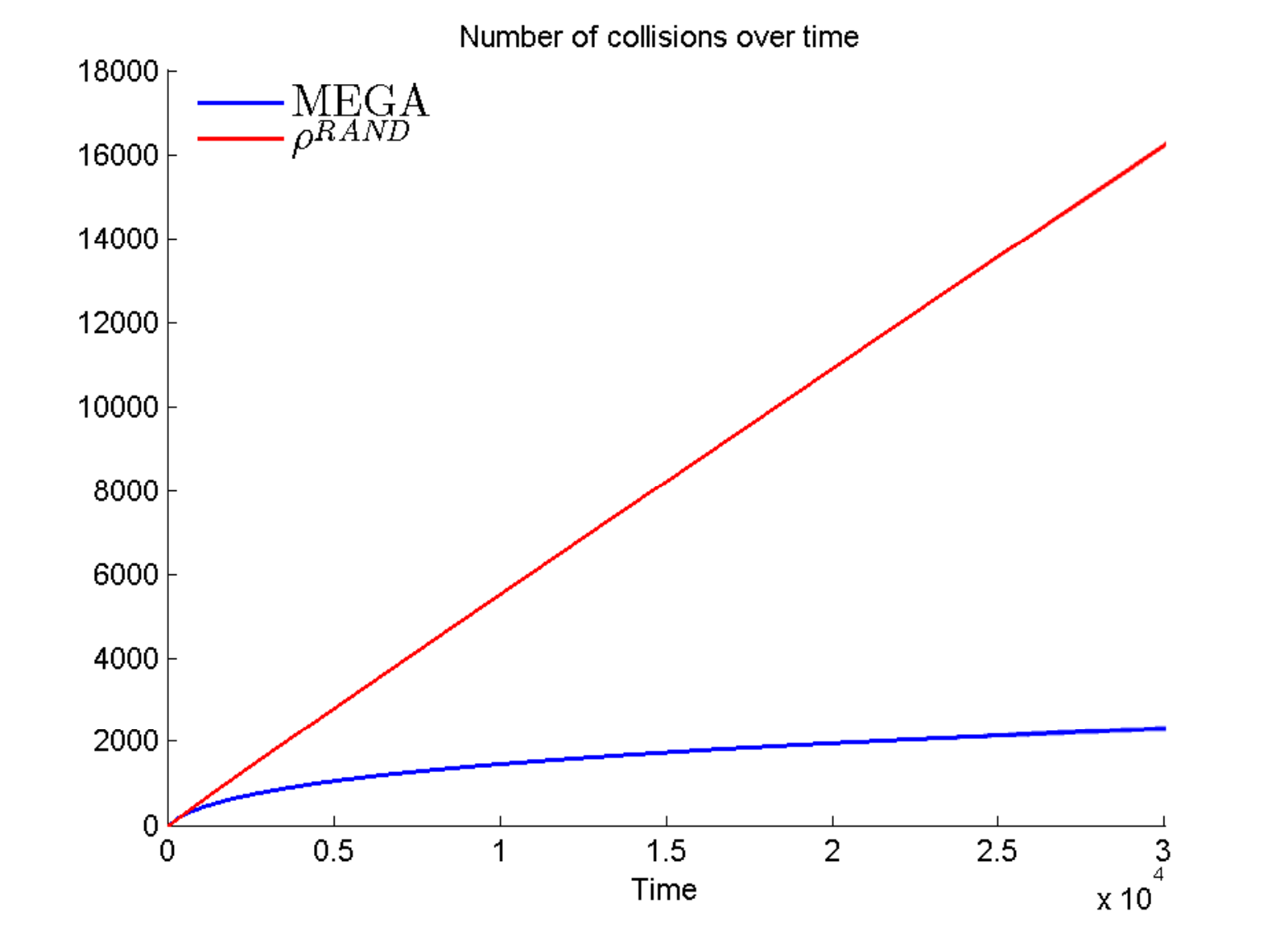}\label{fig:collisions_fixedKN}}
  \caption{Performance of MEGA compared to $\rho^{\text{RAND}}$. The shaded area around the line plots represents result variance over 50 repetitions. It is barely visible due to the small variance. The experiment was run with $N=12$ users and $K=12$ channels, and the number of collisions was averaged over all users.}
\end{center}
\end{figure}

Next, we display the results of experiments in which the number of users changes over time. Initially, the number of users is $1$, gradually increasing until it is equal to $4$, decreasing back to $1$ again. Since $\rho^{\text{RAND}}$ needs a fixed value for the number of users, we gave it the value $N_0=2$, which is the average number of users in the system over time. For different values of $N_0$ the performance of $\rho^{\text{RAND}}$ was rather similar; we present a single value for the sake of clarity.

As before, \figref{fig:regret_arr} displays the average regret over time and \figref{fig:collisions_arr} displays the cumulative number of collisions over time, averaged over all users.

Clearly, MEGA exhibits better performance in terms of regret and collision rate for both scenarios. The significant improvement in the variance (represented by the shaded area around the line plots) of MEGA compared to $\rho^{\text{RAND}}$ is also noteworthy.

\begin{figure}\label{fig:variable}
\begin{center}
  \subfloat[Average regret over time]{\includegraphics[width=0.5\textwidth]{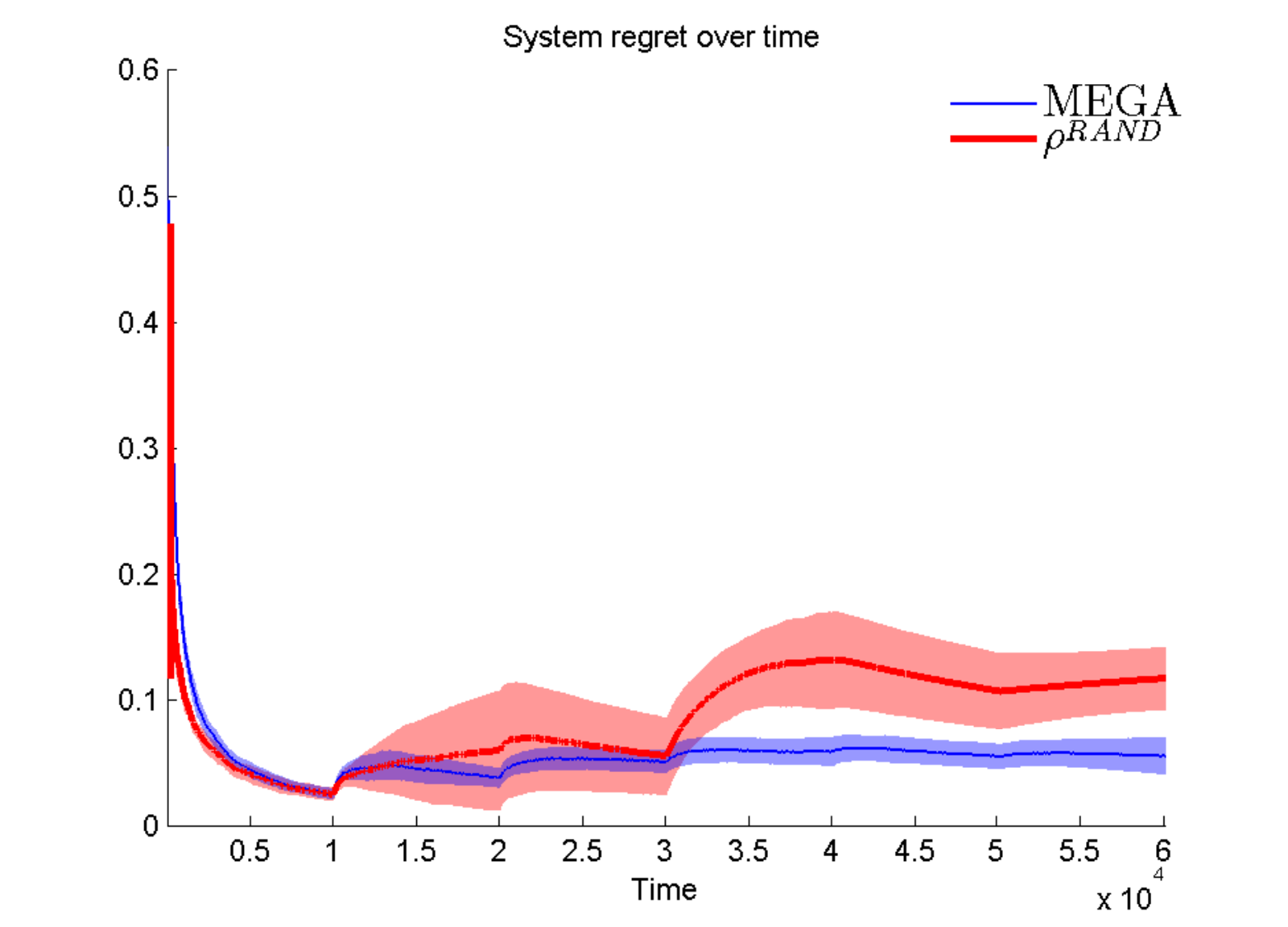}\label{fig:regret_arr}}
  \subfloat[Collisions over time]{\includegraphics[width=0.5\textwidth]{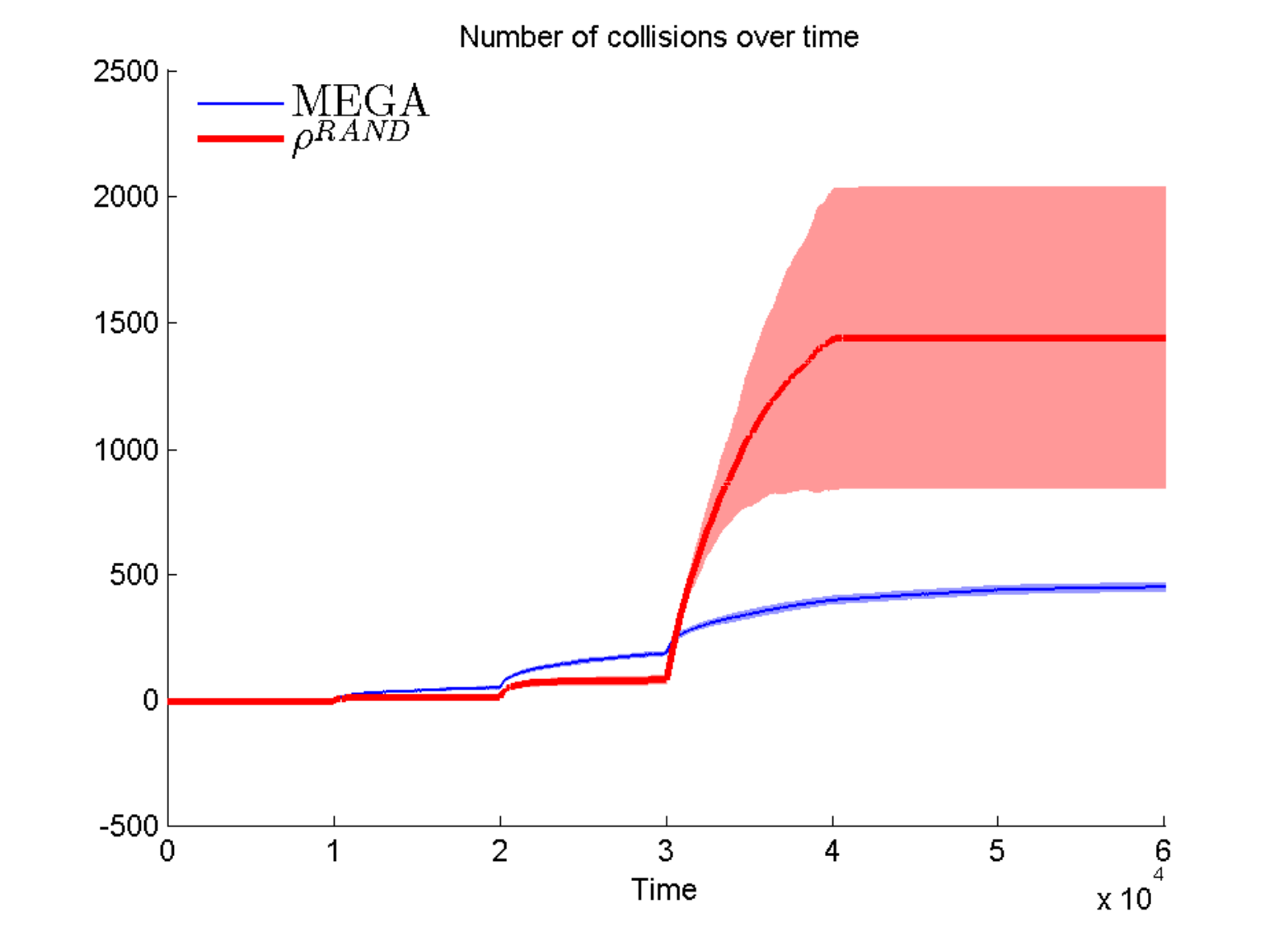}\label{fig:collisions_arr}}
  \caption{Performance of MEGA compared to $\rho^{\text{RAND}}$ in the dynamic scenario. The shaded area around the line plots represents result variance over 20 repetitions. The experiment was run $K=12$ channels, and the number of collisions was averaged over all users.}
\end{center}
\end{figure}

\section{Conclusion}
\label{sec:conclusion}
We formulate the problem of multiple selfish users learning to split the resources of a multi-channel communication system modeled by a stochastic MAB. Our proposed algorithm, a combination of an $\e$-greedy policy with an availability detection mechanism, exhibits good experimental results for both fixed and dynamic numbers of users in the network. We augment these results with a theoretical analysis guaranteeing sub-linear regret. It is worth noting that this algorithm is subject to a very strict set of demands, as mentioned in Sections 1 and 2.

We plan to look into additional scenarios of this problem. For example, an explicit collision indication isn't always available in practice. Also, collisions may result in partial, instead of zero, reward. Another challenge is presented when different users have different views of the arms' characteristics (i.e., receive different rewards). We believe that since our algorithm does not involve communication between users, the different views might actually result in fewer collisions. We would also like to expand our theoretical analysis of the scenario in which the number of users is dynamic, deriving concrete regret bounds for it.

% % ---- Bibliography ---- %
\bibliographystyle{plain}
\bibliography{ecml2014bib}

\begin{thebibliography}{10}

\bibitem{Anandkumar2011}
A.~Anandkumar, N.~Michael, A.K. Tang, and A.~Swami.
\newblock Distributed algorithms for learning and cognitive medium access with
  logarithmic regret.
\newblock {\em Selected Areas in Communications, IEEE Journal on},
  29(4):731--745, 2011.

\bibitem{Auer2002a}
P.~Auer, N.~Cesa-Bianchi, and P.~Fischer.
\newblock Finite-time analysis of the multiarmed bandit problem.
\newblock {\em Machine learning}, 47(2):235--256, 2002.

\bibitem{Auer2002b}
P.~Auer, N.~Cesa-Bianchi, Y.~Freund, and R.E. Schapire.
\newblock The nonstochastic multiarmed bandit problem.
\newblock {\em SIAM Journal on Computing}, 32(1):48--77, 2002.

\bibitem{Auer2010}
P.~Auer and R.~Ortner.
\newblock {UCB} revisited: Improved regret bounds for the stochastic
  multi-armed bandit problem.
\newblock {\em Periodica Mathematica Hungarica}, 61(1):55--65, 2010.

\bibitem{Avner2011}
O.~Avner and S.~Mannor.
\newblock Stochastic bandits with pathwise constraints.
\newblock In {\em 50th IEEE Conference on Decision and Control}, December 2011.

\bibitem{Avner2012}
O.~Avner, S.~Mannor, and O.~Shamir.
\newblock Decoupling exploration and exploitation in multi-armed bandits.
\newblock In {\em 29th International Conference on Machine Learning}, December
  2012.

\bibitem{Berry1985}
D.A. Berry and B.~Fristedt.
\newblock {\em Bandit problems: sequential allocation of experiments}.
\newblock Chapman and Hall London, 1985.

\bibitem{Choe2009}
S.~Choe.
\newblock Performance analysis of slotted aloha based multi-channel cognitive
  packet radio network.
\newblock In {\em Proceedings of the 6th IEEE Conference on Consumer
  Communications and Networking Conference}, CCNC'09, pages 672--676, 2009.

\bibitem{CoverThomas1991}
T.M. Cover and J.A. Thomas.
\newblock {\em Elements of information theory}.
\newblock John Wiley \& Sons, 1991.

\bibitem{Even2002}
E.~Even-Dar, S.~Mannor, and Y.~Mansour.
\newblock {PAC} bounds for multi-armed bandit and markov decision processes.
\newblock In {\em Computational Learning Theory}, pages 193--209. Springer,
  2002.

\bibitem{Garivier2011}
A.~Garivier and O.~Cappé.
\newblock The {KL-UCB} algorithm for bounded stochastic bandits and beyond.
\newblock In {\em Conference On Learning Theory}, pages 359--376, Jul. 2011.

\bibitem{HLP1988}
G.H. Hardy, J.E. Littlewood, and G.~Polya.
\newblock {\em Inequalities}.
\newblock Cambridge University Press, 1988.

\bibitem{Jouini2010}
W.~Jouini, D.~Ernst, C.~Moy, and J.~Palicot.
\newblock Multi-armed bandit based policies for cognitive radio's decision
  making issues.
\newblock In {\em Signals, Circuits and Systems (SCS), 2009 3rd International
  Conference on}, pages 1--6. IEEE, 2010.

\bibitem{Kalathil2012b}
D.~Kalathil, N.~Nayyar, and R.~Jain.
\newblock Decentralized learning for multi-player multi-armed bandits.
\newblock In {\em 51st IEEE Conference on Decision and Control}, pages
  3960--3965, 2012.

\bibitem{Leith2012}
D.J. Leith, P.~Clifford, V.~Badarla, and D.~Malone.
\newblock {WLAN} channel selection without communication.
\newblock {\em Computer Networks}, 2012.

\bibitem{Li2012}
X.~Li, H.~Liu, S.~Roy, J.~Zhang, P.~Zhang, and C.~Ghosh.
\newblock Throughput analysis for a multi-user, multi-channel aloha cognitive
  radio system.
\newblock {\em Wireless Communications, IEEE Transactions on},
  11(11):3900--3909, 2012.

\bibitem{Liu2010}
K.~Liu and Q.~Zhao.
\newblock Distributed learning in multi-armed bandit with multiple players.
\newblock {\em Signal Processing, IEEE Transactions on}, 58(11):5667--5681,
  2010.

\bibitem{Madansky1959}
Albert Madansky.
\newblock Bounds on the expectation of a convex function of a multivariate
  random variable.
\newblock {\em The Annals of Mathematical Statistics}, pages 743--746, 1959.

\bibitem{Mckinney1966}
E.H. McKinney.
\newblock Generalized birthday problem.
\newblock {\em American Mathematical Monthly}, pages 385--387, 1966.

\bibitem{Mitola1999}
J.~Mitola and G.Q. Maguire.
\newblock Cognitive radio: making software radios more personal.
\newblock {\em Personal Communications, IEEE}, 6(4):13 --18, August 1999.

\bibitem{Nie2006}
N.~Nie and C.~Comaniciu.
\newblock Adaptive channel allocation spectrum etiquette for cognitive radio
  networks.
\newblock {\em Mobile Networks and Applications}, 11(6):779--797, December
  2006.

\bibitem{Niyato2008}
D.~Niyato and E.~Hossain.
\newblock Competitive spectrum sharing in cognitive radio networks: a dynamic
  game approach.
\newblock {\em Wireless Communications, IEEE Transactions on}, 7(7):2651
  --2660, July 2008.

\end{thebibliography}

\newpage

\appendix

\section{Supplementary material: proofs}
\label{sec:appendix}

\subsection{Proof of \lemref{lem:collisions}}
Focusing on a single arm $k$, the following observation holds: a period of consecutive collisions (i.e.,  collision event) is always followed by a ``quiet'' period, in which at least one of the colliding users deems $k$ unavailable. Thus, no collisions will occur during this period. In order to bound the expected number of collisions, $\mE{C\paren{t}}$, we examine the series of collision events and collision-free periods.
First, we bound the expected length of a collision event, both from below and from above. We will user the upper bound to bound the number of collisions, and the lower bound to estimate the length of all series up till time $t$. Given a single collision has occurred, the occurrence of the next collision is geometrically distributed with a success probability of $1-p_1p_2$, where $p_1$ and $p_2$ are the persistence probabilities of the colliding users. Thus, the mean length of a collision event is $1 + \frac{p_1p_2}{1-p_1p_2}$.

The persistence probabilities of the two users, $p_1,p_2$, can be bounded in the following manner: for at least one of the users, the first collision is a ``fresh'' attempt at the arm being sampled. Therefore, at the beginning of any collision series, either $p_1=p_0$ or $p_2=p_0$ (or, possibly, both). Also, $p_1 \leq 1$ and $p_2 \leq 1$ at all times. Thus, it is clear that $p_0^2 \leq p_1p_2 \leq p_0$ at all times for any arm.
As a result, we have that the expected length of a collision series, denoted by $\mE{L_c}$, is bounded from above and from below:
\begin{align}\label{eq:L_bounds}
  \mE{L_C} &= 1 + \frac{p_1p_2}{1-p_1p_2} \leq 1 + \frac{p_0}{1-p_0} = \frac{1}{1-p_0},\\
  \mE{L_C} &= 1 + \frac{p_1p_2}{1-p_1p_2} \geq 1 + \frac{p_0^2}{1-p_0^2} = \frac{1}{1-p_0^2}.
\end{align}
From now on, we denote $L_{\text{up}} \triangleq \frac{1}{1-p_0}, L_{\text{low}} \triangleq \frac{1}{1-p_0^2}$.

According to \algref{alg:alg1}, once an agent gives up on an arm, she marks it unavailable for a period of mean length $\half t^\beta$. Thus, the collision-no collision series is of the following form:
\begin{align*}
  L_1, Q_1, L_2, Q_2,\ldots,L_m,Q_m,\ldots,
\end{align*}
where $L_m$ are the collision episodes and $Q_m$ are the ``quiet'' intervals.

We would like to bound the expected number of collisions up till time $t$. In order to do so, we need to calculate $m$ - the expected number of collision \emph{episodes} up till time $t$. This requires calculating the expected length of $L_m$ and $Q_m$.

Starting with $\mE{L_m}$, we will take its lower bound, $L_{\text{low}}$, introduced in \eqref{eq:L_bounds}. Using the lower bound will enable us to devise an upper bound on $m$. Bounding the expectation of $Q_m$ is a bit more complicated:
\begin{align}\label{eq:Q_m}
  \mE{Q_m} = \half\mE{\paren{\sum_{i=1}^{m}L_i + \sum_{i=1}^{m-1}Q_i}^\beta}.
\end{align}
Since $f\paren{x} = x^\beta$ is concave for $\beta<1$, we need to use the Edmunson-Madansky inequality in order to take the expectation of the random variables inside the power operation. According to \cite{Madansky1959}, the expectation of a concave function $g\paren{X}$ of a random variable $X$ defined on a bounded interval $\sbrk{x_1,x_2}$ can be bounded using the first moment of the random variable $X$ in the following manner:
\begin{align*}
  \mE{g\paren{X}} \geq \frac{g\paren{x_2} - g\paren{x_1}}{x_2 - x_1}\paren{\mE{X} - x_1} + g\paren{x_1}.
\end{align*}
In our case, since $\mE{Q_m} = \half t^\beta$, the inequality takes on the following form:
\begin{align*}
  \mE{Q_m} \geq \half t^{\beta-1}\paren{\sum_{i=1}^{m}\mE{L_i} + \sum_{i=1}^{m-1}\mE{Q_i}}.
\end{align*}
Plugging in the lower bound on $L_i$ we have
\begin{align*}
  \mE{Q_m} \geq \half t^{\beta-1}\paren{mL_{\text{low}} + \sum_{i=1}^{m-1}\mE{Q_i}}.
\end{align*}
In order to continue, we use the inequality to extract a bound on $\mE{Q_i}$ for some $i\in\paren{1,\ldots,m-1}$:
\begin{align}\label{eq:Q_i}
  \mE{Q_i} \geq \half t^{\beta-1}i L_{\text{low}}.
\end{align}
In general, replacing $t$, which corresponds to $m$, by a smaller value $t'$, which corresponds to $i$, will yield a tighter bound. However, the bound in \eqref{eq:Q_i} is sufficient for our proof.

We continue developing the bound on $\mE{Q_m}$:
\begin{align*}
  \mE{Q_m}
  &\geq \half t^{\beta-1}\paren{mL_{\text{low}}
    + \half t^{\beta-1}L_{\text{low}}\sum_{i=1}^{m-1}i} \\
  & = \half t^{\beta-1}\paren{mL_{\text{low}}
    + \half t^{\beta-1}L_{\text{low}}\frac{m^2-m}{2}} \\
  &\geq \frac{1}{8} t^{2\beta-2}L_{\text{low}}m^2.
\end{align*}

We now return to $t$, the time interval containing $m$ collision events.
\begin{align*}
  t &= \sum_{i=1}^{m-1}L_i + \sum_{i=1}^{m-1}Q_i\\
    &\geq m L_{\text{low}} + \half t^{\beta-1}L_{\text{low}}\sum_{i=1}^{m-1}i \\
    & = m L_{\text{low}} + \frac{1}{4}t^{\beta-1}L_{\text{low}}\paren{m^2-m} \\
    & \geq \frac{1}{4}t^{\beta-1}L_{\text{low}}m^2,
\end{align*}
which means that
\begin{align*}
  m \leq \frac{2}{\sqrt{L_{\text{low}}}}t^{1-\beta/2}.
\end{align*}

This bound on $m$, the number of collision episodes, enables us to bound the expected number of pairwise, per-channel collisions up till time $t$:
\begin{align}\label{eq:CtBound}
  \mE{C_p\paren{t}} \leq \frac{2L_{\text{up}}}{\sqrt{L_{\text{low}}}}t^{1-\beta/2}.
\end{align}
This number of collisions is sub-linear for any $\beta>0$.
\qed

\subsection{Proof of \lemref{lem:avail}}
The worst case in terms of regret due to lack of availability occurs when all agents ``agree'' on the identity of the best arms, i.e., have converged to the same ranking of arms. Such a situation will lead to a maximal number of collisions as $\e_t$ becomes smaller. We therefore calculate the regret for the scenario of two users targeting some arm $k$, which they deem to be the best arm. The general availability-regret bound is calculated based on this scenario.
Availability regret is accumulated when both agents declare an arm as unavailable simultaneously at the end of a collision streak. A simultaneous cession leads to the series of events depicted in \figref{fig:sim_cession}.

During the time interval $\left[t,t_1\right)$ both agents consider arm $k$ to be taken. At time $t_1$, agent 1 (w.l.o.g) ``unfreezes'' the arm and may begin sampling it. Until time $t_2$, when agent 2 ``unfreezes'' the arm as well, agent 1 may sample it without being disturbed by agent 2. We denote the length of the interval between these events by $\tau \triangleq t_2 - t_1$.

Once both agents consider the arm available, it is only a matter of time until they collide again. The time of their first collision is $t_c$, and the time interval preceding it is $\Delta \triangleq t_c-t_2$. Once a collision occurs both agents continue sampling the arm for a period of length $\ell \triangleq t_g-t_c$, where $t_g$ is the time that one or both agents declare the arm as unavailable once again.
\begin{figure}[ht]
\begin{center}
\centerline{\includegraphics[width=0.5\columnwidth]{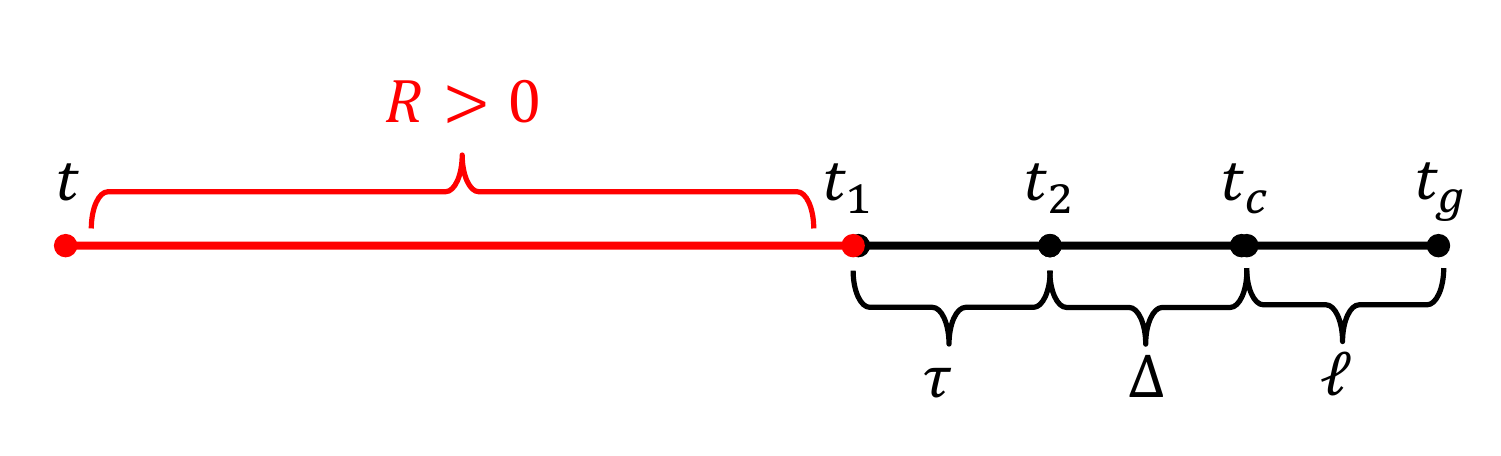}}
\caption{Events following simultaneous cession.}
\label{fig:sim_cession}
\end{center}
\end{figure}

We begin by deriving a lower bound on the probability of an event we denote by $A$: a unilateral cession, given that the previous cession was simultaneous.

Given that a simultaneous cession occurred at some time $t$, both agents' recovery times, $t_1$ and $t_2$, are uniformly distributed in the same manner:
\begin{align*}
  t_1 &\sim U\paren{\sbrk{1, t^\beta}} + t \\
  t_2 &\sim U\paren{\sbrk{1, t^\beta}} + t.
\end{align*}
To keep our notation simple, we assume that $t^\beta$ is an integer, or it is replaced by $\floor{t^\beta}$.

The probability density function of $\tau \triangleq t_2 - t_1$ can be derived as follows:
\begin{align*}
  \forall \tau'\in\set{0,\ldots,t^\beta-1},\;\;
   f_\tau \paren{\tau'} = \sum_{\ell=-\infty}^{\infty} f_{t_1}\paren{\ell}f_{t_2}\paren{\ell+\tau'} = \sum_{\ell=1}^{t^\beta}\frac{1}{t^\beta}f_{t_2}\paren{\ell+\tau'}
  = \sum_{\ell=1}^{t^\beta-\tau'}\frac{1}{t^{2\beta}} = \frac{t^\beta - \tau'}{t^{2\beta}},
\end{align*}
where we used the fact that $ f_{t_i}\paren{\ell} = \sfrac{1}{t^\beta}, \; \forall \ell\in\sbrk{1,t^\beta}$ and derived the probability for the case that $t_2 \geq t_1$.

The probability of the event $A$ depends on all of the variables $t_1, t_2, t_c, t_g$ and on $t$. Alternatively, it depends on the lengths of the intervals between events, $\tau, \delta, \ell$ and on $t$. Formally:
\begin{align}
\label{eq:full_A}
  \mP{A} = \sum_{\tau' = -t^\beta+1}^{t^\beta-1} \sum_{\delta' = 0}^{\infty} \sum_{\ell' = 1}^{\infty} \cP{A}{\tau = \tau', \Delta = \delta', \ell = \ell'}\mP{\tau = \tau', \Delta = \delta', \ell = \ell'}.
\end{align}
We begin by examining the dependency on $\tau$. From symmetry we have that
\begin{align*}
  \mP{A} &= \sum_{\tau' = -t^\beta+1}^{t^\beta-1} \cP{A}{\tau = \tau'}\mP{\tau = \tau'}\\
  &=  2 \sum_{\tau' = 1}^{t^\beta-1} \cP{A}{\tau = \tau'}\mP{\tau = \tau'} + \cP{A}{\tau = 0}\mP{\tau = 0} \\
  &\geq 2 \sum_{\tau' = 1}^{t^\beta-1} \cP{A}{\tau = \tau'}\mP{\tau = \tau'}.
\end{align*}
In addition, for any givne $t$, the distribution of $\tau$ is independent of that of $\Delta$ and $\ell$.

Next, we observe that the dependency of the event $A$ on $\tau$ and $\Delta$ is through the number of successful sample attempts agent 1 makes before agent 2 interferes. This number of samples can be bounded using a binomial random variable with $\tau + \Delta$ trials and a success probability $p = 1-\e_t$. The last observation can be justified as follows: for every $t'>T$, the colliding agents rank some arm $k$ as optimal. If it is considered available, they sample it with probability $1 - \e_t' + \sfrac{\e_t'}{K_A}$, where $K_A$ is the number of arms available to each agent. This probability monotonously increases over time and is lower bounded by $p = 1-\e_t$ for any $t'\in\sbrk{t, t_g}$.

Examining the distribution of $\Delta$, we see that it is lower bounded by a geometrically distributed random variable with parameter $q = \paren{1-\e_t}^2$, following an argument similar to the one above. In our analysis we will use the following bound:
\begin{align}
\label{eq:A_delta}
  \mP{A} = \sum_{\delta' = 0}^{\infty} \cP{A}{\Delta = \delta'}\mP{\Delta = \delta'} \geq \cP{A}{\Delta = 0}\paren{1-\e_t}^2.
\end{align}

Finally, the event $A$ depends on the length of the collision series, $\ell$. Since we assume that $\Delta = 0$ in our bound (see \eqref{eq:A_delta}), the persistence probability of agent 2 is simply $p_0$, as it does not accumulate any successful sample attempts. The persistence probability of agent 1, on the other hand, depends on the number of sample attempts it made over the course of $\tau$ rounds. When two agents with persistence probabilities $p_1$ and $p_2$ collide, the probability of agent 1 ``conquering'' the collision streak is
\begin{align*}
  \mP{A'} = \sum_{\ell' = 1}^{\infty} p_1^{\ell'} p_2^{\ell'-1}\paren{1-p_2} = p_1\frac{1-p_2}{1-p_1p_2}.
\end{align*}
The expression for $p_1$ as a function of the number of successful samples agent 1 made, $m_1$, is $p_1 = 1 - \alpha^{m_1}\paren{1-p_0}$. Plugging in the fact that $p_2 = p_0$ and this expression yields
\begin{align*}
  \mP{A'\paren{m_1}} = 1 - \frac{\alpha^{m_1}}{1 + \alpha^{m_1} p_0}.
\end{align*}

Combining the observations above, we re-write \eqref{eq:full_A}:
\begin{align*}
  \mP{A} \geq 2\sum_{\tau' = 1}^{t^\beta-1}\sum_{m = 0}^{\tau} \mP{A'\paren{m_1\paren{\tau = \tau',\Delta = 0}}} \mP{\tau = \tau'} \mP{\Delta = 0} \mP{m_1 = m}.
\end{align*}
We note that $\mP{A}\geq \mP{A'}$, since $A'$ does not include the probability of agent 2 conquering the collision streak.

Further developing our bound, we have that
\begin{align}
\label{eq:A_dev}
\begin{split}
  \mP{A} &\geq 2\sum_{\tau' = 1}^{t^\beta-1}\sum_{m = 0}^{\tau'} \mP{A'\paren{m_1\paren{\tau = \tau',\Delta = 0}}} \mP{\tau = \tau'} \mP{\Delta = 0} \mP{m_1 = m}\\
    &\geq 2\paren{1-\e_t}^2\sum_{\tau' = 1}^{t^\beta-1}\frac{t^\beta - \tau'}{t^{2\beta}}\sum_{m = 0}^{\tau'}
    \mP{m_1 = m} \paren{1 - \frac{\alpha^{m}}{1 + \alpha^{m} p_0}} \\
    &\geq 2\paren{1-\e_t}^2\sum_{\tau' = 1}^{t^\beta-1}\frac{t^\beta - \tau'}{t^{2\beta}}\sum_{m = 0}^{\tau'}
    \mP{m_1 = m} \paren{1 - \alpha^{m}} \\
    &=  2\paren{1-\e_t}^2\sbrk{\sum_{\tau' = 1}^{t^\beta-1}\frac{t^\beta - \tau'}{t^{2\beta}}\sum_{m = 0}^{\tau'} \mP{m_1 = m} -
    \sum_{\tau' = 1}^{t^\beta-1}\frac{t^\beta - \tau'}{t^{2\beta}}\sum_{m = 0}^{\tau'} \mP{m_1 = m} \alpha^{m}} \\
    &= 2\paren{1-\e_t}^2\sbrk{\underbrace{\sum_{\tau' = 1}^{t^\beta-1}\frac{t^\beta - \tau'}{t^{2\beta}}}_{A_1} -
    \underbrace{\sum_{\tau' = 1}^{t^\beta-1}\frac{t^\beta - \tau'}{t^{2\beta}}\sum_{m = 0}^{\tau'} \mP{m_1 = m} \alpha^{m}}_{A_2}},
\end{split}
\end{align}
where the last equality stems from the fact that the support of the random variable $m_1$ is $\set{0,\ldots,\tau}$.

We now address the two terms in \eqref{eq:A_dev}, denoted by $A_1$ and $A_2$, separately.

We would like to link the inner sum of $A_2$ to the moment generating function (MGF) of a binomial random variable.  By definition, the moment generating function of a random variable $X$ is $M_X\paren{t} = \mE{e^{tX}}$. Using the fact that $\alpha^m = e^{m\ln\alpha}$, we have that the inner sum of $A_2$ corresponds to $M_X\paren{\ln\alpha}$. In order to upper bound $A_2$, we need to examine the dependency of the moment generating function on the success probability. In our case, the success probability is time-dependent through the exploration factor $\e_t$. Thus, $m_1$ is not a binomial random variable. However, the success probability can be bounded using the extreme values of $t_1$ and $t_2$. The time interval during which the ``pseudo-binomial'' experiment takes place is $\sbrk{t_1,t_2}$. The minimal value of $t_1$ is $t+1$, while the maximal value of $t_2$ is $t+t^\beta$. Thus, $\e_{t+1} \geq \e_{t'} \geq \e_{t+t^\beta}$ for any $t'\in\sbrk{t_1,t_2}$ ($\e_t$ is monotonously decreasing in $t$). The success probability of each trial in our algorithm is $p_{t'} = 1-\e_{t'} + \sfrac{\e_{t'}}{K_A} \geq 1-\e_{t+1}$.
%Returning to the inner sum in $A_2$, we have that
%\begin{align*}
%  \sum_{m = 0}^{\tau'} \mP{m_1 = m} \alpha^{m} \leq \sum_{m = 0}^{\tau'} p \alpha^{m}.
%\end{align*}
%We stress that the ``bounding'' success probability $p$ is fixed for each episode of post-collision recovery.
The MGF of a binomial random variable $X\sim\text{B}\paren{n,p}$ is $M_X\paren{t} = \paren{1-p+p e^t}^n$, which is transformed into $f\paren{p} = \paren{1-p+\alpha p}^n$ in our case. We examine $f\paren{p\paren{t'}}$ for $t'\in\sbrk{t_1,t_2}$:
\begin{align*}
  f'\paren{p} = \frac{df}{dp}\frac{dp}{dt'} = n\paren{\alpha - 1}\sbrk{1-p\paren{1-\alpha}}^{n-1}\frac{c}{t'^2},
\end{align*}
which is negative for all $t'$ ($c$ is a constant determined by our algorithm). We conclude that maximum of $f$ is obtained when $t'$ is minimal, resulting in the following bound:
\begin{align}\label{eq:Bin_sum}
  \sum_{m = 0}^{\tau'} \mP{m_1 = m} \alpha^{m} \leq  \paren{1 - \paren{1-\e_{t+1}}\paren{1-\alpha}}^{\tau'}.
\end{align}
Denoting $q = 1 - \paren{1-\e_{t+1}}\paren{1-\alpha}$, we have that
\begin{align*}
  A_2 &\leq \sum_{\tau' = 1}^{t^\beta-1}\frac{t^\beta - \tau'}{t^{2\beta}} q^{\tau'} \leq \frac{1}{t^{\beta}}\sum_{\tau' = 1}^{t^\beta-1} q^{\tau'}
      \leq \frac{1}{t^{\beta}}\cdot \frac{q - q^{t^{\beta}-1}}{1 - q} \leq \frac{1}{t^{\beta}}\cdot \frac{q}{1 - q} \\
      &= \frac{1}{t^{\beta}}\cdot \frac{1 - \paren{1 - q}}{1 - q} = \frac{1}{t^{\beta}}\cdot \frac{1}{1 - q} - \frac{1}{t^{\beta}}
      \leq \frac{1}{t^{\beta}}\cdot \frac{1}{\paren{1-\e_{t+1}}\paren{1-\alpha}}.
\end{align*}
Since $\e_t$ decreases over time, we can define $T_2 = \min_t\set{t:\e_{t+1} \leq 0.5}$, resulting in a simple bound for $A_2$:
\begin{align*}
  A_2 \leq \frac{2}{1-\alpha}\cdot \frac{1}{t^{\beta}},
\end{align*}
for all $t > T_2$.

Next, we examine $A_1$:
\begin{align*}
  A_1 = \sum_{\tau' = 1}^{t^\beta-1}\frac{t^\beta - \tau'}{t^{2\beta}} = 1 - \frac{1}{t^\beta} - \frac{1}{t^{2\beta}}\sum_{\tau' = 1}^{t^\beta-1}\tau'
  = 1 - \frac{1}{t^\beta} - \frac{t^\beta-1}{2t^\beta} = \half - \frac{1}{2t^\beta}.
\end{align*}

Combining the last results, we have
\begin{align}\label{eq:A}
  \mP{A} \geq \paren{1-\e_t}^2\paren{1 - \frac{5 - \alpha}{1-\alpha}\cdot\frac{1}{t^\beta}}, \;\forall t > T_2.
\end{align}

We continue by deriving an upper bound on the probability of the event $B$: a simultaneous cession, given that the previous cession was unilateral.
\begin{figure}[ht!]
\vskip 0.2in
\begin{center}
\centerline{\includegraphics[width=0.5\columnwidth]{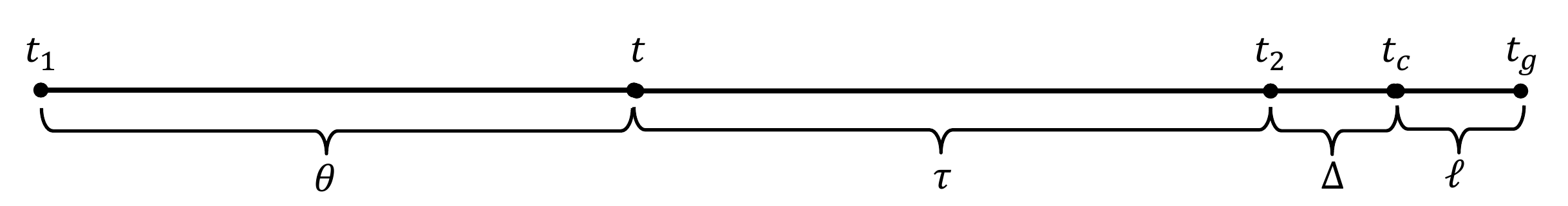}}
\caption{Events preceding and following unilateral cession.}
\label{fig:unilateral_cession}
\end{center}
\vskip -0.2in
\end{figure}

The chain of events matching $B$ is displayed in \figref{fig:unilateral_cession}; our notation is similar to that used in the analysis of the event $A$. $t_1$ is the last point in time in which agent 1 ``unfreezes'' the arm under dispute. She later goes on to sample it, collide with agent 2 and conquer the collision streak, which ends at time $t$ with a unilateral cession of agent 2. The time interval $\theta \triangleq \sbrk{t_1,t}$ includes a period of unknown length during which agent 1 increases her persistence by successfully sampling arm $k$. This persistence is \emph{not} reset at time $t$, since agent 1 does not cede its right to sample the arm - the collision streak ends due to agent 2's unilateral cession. At time $t$ agent 2 steps down, allowing agent 1 to continue sampling arm $k$ without being disturbed. Agent 2 considers arm $k$ to be unavailable all through the interval $\tau \triangleq \left(t,t_2\right]$. At time $t_2$ agent 2 ``unfreezes'' the arm, and at time $t_c$ it begins a new collision streak with agent 1. The time interval until the collisions begin is denoted by $\Delta$, the length of the collision streak is $\ell$ and its end is mark by time $t_g$.

The characteristics of the collision streak which begins at $t_c$, from agent 1's point of view, are determined by both the history of sampling in the interval $\theta$ and the history of sampling in the interval $\tau$. However, we argue that for large values of $t$ the interval $\tau$ is long enough for us to derive a tight bound on $\mP{B}$ without analyzing the events in $\theta$.

Before we examine $\mP{B}$ directly, we note that it depends only on the length of the intervals $\theta$ and $\tau$, through the number of times agent 1 samples arm $k$ during these intervals. This is a random variable which resembles a Binomial random variable, albeit with a temporally dependent success probability. We denote this variable by $m_1\paren{\theta}$ and $m_1\paren{\tau}$ for each of the intervals. As before, we assume $t^\beta$ and $t^{\beta/2}$ to be integers, or replaced by rounded values.

We begin by examining the range of values of $\tau$, and its effect on $\mP{B}$.
\begin{align*}
  \mP{B} = \sum_{\tau' = 1}^{t^\beta}\cP{B}{\tau = \tau'}\mP{\tau = \tau'}
    =  \underbrace{\sum_{\tau' = 1}^{t^{\beta/2}-1}\cP{B}{\tau = \tau'}\mP{\tau = \tau'}}_{B_1}
    + \underbrace{\sum_{\tau' = t^{\beta/2}}^{t^{\beta}}\cP{B}{\tau = \tau'}\mP{\tau = \tau'}}_{B_2}
\end{align*}

Bounding $B_1$ is rather simple:
\begin{align*}
  B_1 = \sum_{\tau' = 1}^{t^{\beta/2}-1}\cP{B}{\tau = \tau'}\mP{\tau = \tau'}
  \leq \sum_{\tau' = 1}^{t^{\beta/2}-1}\mP{\tau = \tau'}
  = \frac{t^{\beta/2}-1}{t^\beta}
  \leq \frac{1}{t^{\beta/2}},
\end{align*}
where we used the fact that $\tau$ is uniformly distributed in $\sbrk{1,t^\beta}$.

Turning to $B_2$, we note that $\cP{B}{\tau = \tau'}$ decays as $\tau$ grows. As a result, we have that
\begin{align*}
  B_2 = \sum_{\tau' = t^{\beta/2}}^{t^{\beta}}\cP{B}{\tau = \tau'}\mP{\tau = \tau'}
  \leq \frac{1}{t^\beta}\sum_{\tau' = t^{\beta/2}}^{t^{\beta}}\cP{B}{\tau = t^{\beta/2}}
  = \frac{t^{\beta} - t^{\beta/2}}{t^\beta}\cP{B}{\tau = t^{\beta/2}}
  \leq \cP{B}{\tau = t^{\beta/2}}.
\end{align*}

We continue addressing the different sources of randomness affecting $\mP{B}$ (more specifically, $B_2$). The length of the interval $\Delta = t_c - t_2$ depends on the sampling policies of both agents. Each agent has a probability of sampling arm $k$ that depends on time and on the number of arms that agent considers available:
\begin{align*}
  P_{s,i} = 1-\e_t' + \frac{\e_t'}{K_{A,i}}, \;\forall t'\geq T.
\end{align*}
The probability distribution of $\Delta$ resembles a geometric distribution with a changing success probability, $1-P_{s,1}P_{s,2}$. Let us analyze the geometric distribution for a fixed sampling probability $p$:
\begin{align*}
  \mP{\Delta = \delta'} = \paren{1-p}^{\delta'}p.
\end{align*}
Differentiating w.r.t $p$, we obtain
\begin{align*}
  \frac{dP}{dp} = -\delta'\paren{1-p}^{\delta'-1}p + \paren{1-p}^{\delta'} = \paren{1-p}^{\delta'-1}\sbrk{-\delta'p  + 1 - p} = \paren{1-p}^{\delta'-1}\sbrk{1 - p\paren{1+\delta'}}.
\end{align*}
This derivative is negative whenever $1 - p\paren{1+\delta'}$ is negative. Defining $T_3 = \min_t\set{t:1-\e_t > 0.5}$, this condition holds for every $\delta'>0,\; t>T_3$. We therefore conclude that $\mP{\Delta = \delta'}$ decreases as $P_{s,i}$ grows, and we can bound it using a bound on $P_{s,i}$:
\begin{align*}
  P_{s,i} \geq 1-\e_t' \geq 1-\e_{t+1},  \;\forall t'> t,
\end{align*}
so that:
\begin{align}
\label{eq:delta_prob}
  \mP{\Delta = \delta'} \leq \paren{1 - \paren{1-\e_{t+1}}^2}^{\delta'}\paren{1-\e_{t+1}}^2, \;\delta'\in\set{0,\infty},\;t>T_3.
\end{align}
This bound allows us to ignore the dependency of $\Delta$ on $\tau$ (which exists since $P_{s,i}$ is time dependent), so that
\begin{align}
\label{eq:B_2}
  B_2 \leq \cP{B}{\tau = t^{\beta/2}} = \sum_{\delta' = 0}^{\infty} \cP{B}{\tau = t^{\beta/2}, \Delta = \delta'}\mP{\Delta = \delta'}.
\end{align}

We now turn to calculating $\cP{B}{\tau = \tau', \Delta = \delta'}$, introducing the following definitions: $m_1\paren{\tau'}$ is a Binomial random variable that describes the number of successful attempts agent 1 made when sampling arm $k$ during the interval $\sbrk{t,t_2}$. Similarly, $m_1\paren{\delta'}$ is the Binomially distributed number of successes of agent 1 in the interval $\left(t_2,t_c\right]$, and $m_2\paren{\delta'}$ is the Binomially distributed number of successes of agent 2 in the interval $\left(t_2,t_c\right]$. We note that by definition $m_2\paren{\tau'} = 0$. Given that the interval $\left(t_2,t_c\right]$ is of length $\Delta = \delta'$, we know that $m_1\paren{\delta'} + m_2\paren{\delta'} \leq \delta'$. This is used in setting the summation boundaries in the equation below.
\begin{align}
\begin{split}
\label{eq:full_B}
  &\cP{B}{\tau = \tau', \Delta = \delta'} \\
  &= \sum_{\mu_1 = 0}^{\delta'} \cP{B}{\tau = \tau', \Delta = \delta', m_1\paren{\delta'} = \mu_1}\mP{m_1\paren{\delta'} = \mu_1} \\
  &= \sum_{\mu_1 = 0}^{\delta'} \sum_{\mu_2 = 0}^{\delta' - \mu_1} \cP{B}{\tau = \tau', \Delta = \delta', m_1\paren{\delta'} = \mu_1, m_2\paren{\delta'} = \mu_2}
    \cP{m_1\paren{\delta'} = \mu_1}{m_2\paren{\delta'} = \mu_2}\mP{m_2\paren{\delta'} = \mu_2} \\
  &= \sum_{\mu_1 = 0}^{\delta'} \mP{m_1\paren{\delta'} = \mu_1} \sum_{\mu_2 = 0}^{\delta' - \mu_1} \mP{m_2\paren{\delta'} = \mu_2} \\
    &\quad\quad\sum_{\eta_1 = 0}^{\tau'} \mP{m_1\paren{\tau'} = \eta_1} \cP{B}{\tau = \tau', \Delta = \delta', m_1\paren{\delta'} = \mu_1, m_2\paren{\delta'} = \mu_2, m_1\paren{\tau'} = \eta_1}.
\end{split}
\end{align}

The expression $\cP{B}{\tau = t^{\beta/2}, \Delta = \delta', m_1\paren{\delta'} = \mu_1, m_2\paren{\delta'} = \mu_2, m_1\paren{\tau'} = \eta_1}$ can be calculated based on the definition of our algorithm.  Given $n_i$ successful samples of an arm, an agent's persistence in a collision event is $p_i = 1 - \alpha^{n_i}\paren{1-p_0}$. Given $p_1$ and $p_2$, the probability of a simultaneous cession is
\begin{align*}
  P_{\text{sim}} = \sum_{\ell = 1}^\infty p_1^{\ell-1}\paren{1-p_1}p_2^{\ell-1}\paren{1-p_2} = \paren{1-p_1}\paren{1-p_2}\sum_{\ell = 0}^\infty \paren{p_1p_2}^\ell = \frac{\paren{1-p_1}\paren{1-p_2}}{1-p_1p_2}.
\end{align*}
In our analysis we have that $n_1 = m_1\paren{\tau'}+m_1\paren{\delta'}$ and $n_2 = m_2\paren{\delta'}$. Plugging in the values $m_1\paren{\delta'} = \mu_1, m_2\paren{\delta'} = \mu_2, m_1\paren{\tau'} = \eta_1$ and manipulating yields:
\begin{align*}
  P_{\text{sim}}
  = \frac{\alpha^{\eta_1 + \mu_1}\alpha^{\mu_2}\paren{1-p_0}^2}{1 - \paren{1 - \alpha^{\eta_1 + \mu_1}\paren{1-p_0}}\paren{1 - \alpha^{\mu_2}\paren{1-p_0}}}
  = \frac{\alpha^{\eta_1 + \mu_1+\mu_2}\paren{1-p_0}}{\alpha^{\eta_1 + \mu_1} + \alpha^{\mu_2} - \alpha^{\eta_1 + \mu_1+\mu_2}\paren{1-p_0}}.
\end{align*}

We turn to simplifying $P_{\text{sim}}$ by bounding it from above:
\begin{align*}
  P_{\text{sim}}
  &= \frac{\alpha^{\eta_1 + \mu_1+\mu_2}\paren{1-p_0}}{\alpha^{\eta_1 + \mu_1} + \alpha^{\mu_2} - \alpha^{\eta_1 + \mu_1+\mu_2}\paren{1-p_0}} \\
  &= \frac{\alpha^{\eta_1 + \mu_1+\mu_2}\paren{1-p_0}}{\alpha^{\eta_1 + \mu_1} +
  \alpha^{\mu_2} - \alpha^{\eta_1 + \mu_1+\mu_2} +  \alpha^{\eta_1 + \mu_1+\mu_2}p_0} \\
  &\leq \frac{\alpha^{\eta_1 + \mu_1+\mu_2}\paren{1-p_0}}{\alpha^{\mu_2} - \alpha^{\eta_1 + \mu_1+\mu_2}} \\
  &= \paren{1-p_0}\frac{\alpha^{\eta_1 + \mu_1}}{1 - \alpha^{\eta_1 + \mu_1}}.
\end{align*}

Next, we show that with high probability, $m_1\paren{\tau'} > 0$. The random variable $m_1\paren{\tau'}$ behaves similarly to a binomial random variable, with the exception of a time-dependent success probability. For times $t'$ such that $t'>t$ and $t'>T$, we have that the success probability of a single trial (e.g. the probability of sampling arm $k$ at a single time step) is bounded: $p_{t'} \geq 1-\e_{t+1} \triangleq p_s$. Therefore,
\begin{align*}
  \mP{m_1\paren{\tau'} > 0} = 1 - \mP{m_1\paren{\tau'} = 0}
  \geq 1 - \paren{1 - p_s}^{\tau'}
  = 1 - \paren{1-\e_{t+1}}^{\tau'} \;\forall t>T.
\end{align*}

% NOTE: this turns the bound into a bound w.h.p.

Based on this observation, we can further develop $P_{\text{sim}}$:
\begin{align*}
  P_{\text{sim}} \leq \frac{1-p_0}{1-\alpha}\alpha^{\eta_1 + \mu_1}.
\end{align*}

Plugging into \eqref{eq:full_B}, we have
\begin{align*}
  \cP{B}{\tau = \tau', \Delta = \delta'}
  &\leq \sum_{\mu_1 = 0}^{\delta'} \mP{m_1\paren{\delta'} = \mu_1}
  \sum_{\mu_2 = 0}^{\delta' - \mu_1} \mP{m_2\paren{\delta'} = \mu_2}
  \sum_{\eta_1 = 0}^{\tau'} \mP{m_1\paren{\tau'} = \eta_1} \frac{1-p_0}{1-\alpha}\alpha^{\eta_1 + \mu_1} \\
  &= \frac{1-p_0}{1-\alpha}\sum_{\mu_1 = 0}^{\delta'} \alpha^{\mu_1}\mP{m_1\paren{\delta'} = \mu_1}
  \sum_{\mu_2 = 0}^{\delta' - \mu_1} \mP{m_2\paren{\delta'} = \mu_2}
  \sum_{\eta_1 = 0}^{\tau'}\alpha^{\eta_1} \mP{m_1\paren{\tau'} = \eta_1} \\
  &= \frac{1-p_0}{1-\alpha}\sum_{\mu_1 = 0}^{\delta'} \alpha^{\mu_1}\mP{m_1\paren{\delta'} = \mu_1}
  \sum_{\eta_1 = 0}^{\tau'}\alpha^{\eta_1} \mP{m_1\paren{\tau'} = \eta_1}.
\end{align*}
We now apply the same analysis as the one used to bound the inner some of $A_2$, the result of which appears in \eqref{eq:Bin_sum}, in order to address the sums over functions of probabilities of the ``pseudo-binomial'' variables $m_1\paren{\tau'}$ and $m_1\paren{\delta'}$. Based on the result in \eqref{eq:Bin_sum}, we further develop $\cP{B}{\tau = \tau', \Delta = \delta'}$:
\begin{align*}
  \cP{B}{\tau = \tau', \Delta = \delta'}
  &\leq \frac{1-p_0}{1-\alpha}\paren{1-\paren{1-\alpha}p_s}^{\delta'}
  \paren{1- \paren{1-\alpha}p_s}^{\tau'} \\
  &\leq \frac{1-p_0}{1-\alpha}
  \paren{1-\paren{1-\alpha}\paren{1-\e_{t+1}}}^{\delta'+\tau'}.
\end{align*}

Combining this result with \eqref{eq:delta_prob} and \eqref{eq:B_2} yields
\begin{align*}
  B_2 \leq \sum_{\delta'=0}^\infty \frac{1-p_0}{1-\alpha}
  \sbrk{1-\paren{1-\alpha}\paren{1-\e_{t+1}}}^{\delta'+t^{\beta/2}}\paren{1-\e_{t+1}}^2
  \sbrk{1 - \paren{1-\e_{t+1}}^2}^{\delta'}.
\end{align*}
Denoting $q = 1-\e_{t+1}$ and rearranging, we have
\begin{align*}
  B_2 &\leq \frac{1-p_0}{1-\alpha}q^2\paren{1 - \paren{1-\alpha}q}^{t^{\beta/2}}
  \sum_{\delta'=0}^{\infty} \sbrk{1 + \paren{1-\alpha}q^3 - q^2 - \paren{1-\alpha}q}^{\delta'} \\
  &\leq \frac{1-p_0}{1-\alpha}q^2\paren{1 - \paren{1-\alpha}q}^{t^{\beta/2}}
  \sum_{\delta'=0}^{\infty} \sbrk{1 - q^2}^{\delta'} \\
  &= \frac{1-p_0}{1-\alpha}\paren{1 - \paren{1-\alpha}q}^{t^{\beta/2}},
\end{align*}
where we used the fact that since $q<1$, $\paren{1-\alpha}q^3 < \paren{1-\alpha}q$.

Combining the bounds on $B_1$ and $B_2$, we finally have that
\begin{align}\label{eq:B}
  \mP{B} \leq \frac{2}{t^{\beta/2}}.
\end{align}

We now use the bounds for $\mP{A}$ and $\mP{B}$ to bound the expected regret introduced by the availability mechanism up to a certain time $t$.
Let us examine the interval $\sbrk{T,t}$. During this period, the bounds for $\mP{A}$ and $\mP{B}$ hold, and vary over time. The regret accumulated following a simultaneous cession also varies from episode to episode, and is determined by the time the cession occurred.

In order to analyze the regret, we divide the period $\sbrk{T,t}$ into equal-length intervals. Setting the interval length to be $L = \paren{t-T}^{\beta/2}$, the number of intervals is $N = \frac{t-T}{\paren{t-T}^{\beta/2}} = \paren{t-T}^{1-\beta/2}$. The regret accumulated during a single interval is trivially bounded by its length, $L$. The probability of acquiring regret during an interval depends on its index, i.e. place among intervals, and can be bounded based on the time the interval began, denoted by $t_j$. For an illustration see \figref{fig:time_intervals}.

\begin{figure}[ht!]
\vskip 0.2in
\begin{center}
\centerline{\includegraphics[width=0.75\columnwidth]{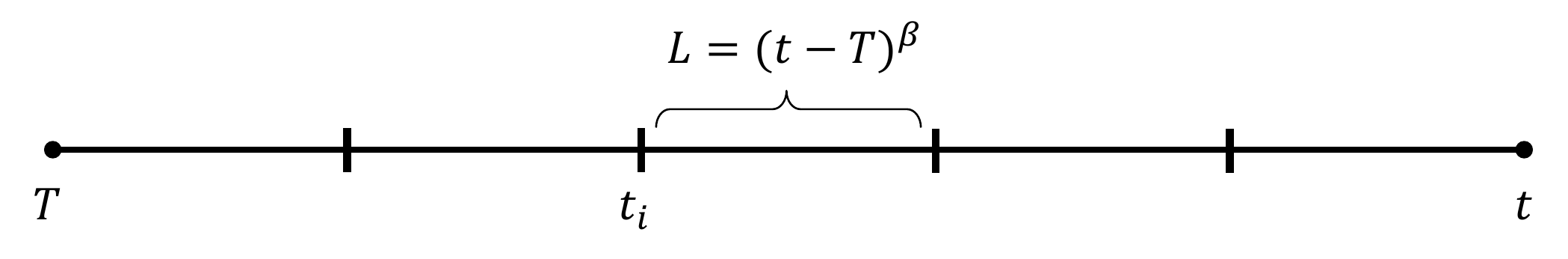}}
\caption{Analyzing regret by dividing time into intervals of equal length.}
\label{fig:time_intervals}
\end{center}
\vskip -0.2in
\end{figure}

The time an interval begins can be simply calculated: $t_j = T + j\paren{t-T}^{\beta/2}$. We bound the probability of a simultaneous cession occurring at $t_j$ using the bounds we derived for $\mP{A}$ and $\mP{B}$, and denote it by $\mP{R_j}$. The bound is based on the fact that (after time T) the way a collision event ends depends only on the sampling history of the episode preceding it.
\begin{align*}
  \mP{R_j} = \paren{1 - \mP{A}}\mP{R_{i-j}} + \mP{B}\paren{1 - \mP{R_{i-j}}} \leq 1 - \mP{A} + \mP{B},
\end{align*}
where $\mP{A}$ and $\mP{B}$ are time dependent, as expressed in \eqref{eq:A} and \eqref{eq:B}.
Developing $1-\mP{A}$:
\begin{align*}
  1-\mP{A}
  &\leq 1 - \paren{1-\e_t}^2\paren{1 - \frac{5 - \alpha}{1-\alpha}\cdot\frac{1}{t^\beta}} \\
  &= 1 - \paren{1-\e_t}^2 + \frac{5 - \alpha}{1-\alpha}\cdot\frac{\paren{1-\e_t}^2}{t^\beta} \\
  &= 2\e_t - \e_t^2 + \frac{5 - \alpha}{1-\alpha}\cdot\frac{\paren{1-\e_t}^2}{t^\beta} \\
  &\leq 2\e_t + \frac{5 - \alpha}{1-\alpha}\frac{1}{t^\beta} \\
  &\leq C_0\frac{1}{t^\beta},
\end{align*}
where $C_0 = \frac{2cK^2}{d\paren{K-1}} + \frac{5 - \alpha}{1-\alpha}$. Combined with the bound on $\mP{B}$, we have that
\begin{align}\label{eq:bound_R_j}
  \mP{R_j} \leq \frac{C_2}{t^{\sfrac{\beta}{2}}},
\end{align}
where $C_2 = C_0 + 2 = \frac{2cK^2}{d\paren{K-1}} + \frac{7 - 3\alpha}{1-\alpha}$.

During the interval $\sbrk{t_j,t_{j+1}}$, more than one collision episode may take place. In order to analyze the regret accumulated during an interval, we examine a simplified (bounding) model: if a collision episode ends with a simultaneous cession at any time during the $j^\text{th}$ interval, the regret accumulated is the entire length of the interval, $L$. Therefore, in order for an interval to contribute zero regret, \emph{all} collision episodes occurring within it must end with unilateral cessions.
Let us define the following random variable:
\begin{align*}
  S_{m,j} = \sum_{i = 1}^m U_i,
\end{align*}
where $U_m$ are random variables that represent the length of the ``quiet'' period following a unilateral cession (this period is denoted by $\tau$ in \figref{fig:unilateral_cession}) and $m$ is the number of consecutive episodes ending with unilateral cessions in the $j^\text{th}$ interval (a realization of a random variable denoted by $M$). In order for an interval to end ``well'' and contribute zero regret, the following event must occur: $S_{m,j} \geq L$ for all possible values of $m$. We can now bound the probability of accumulating regret during the $j^\text{th}$ interval, denoted by $C_j$:
\begin{align}\label{eq:Regret_Int}
  \mP{C_j} \triangleq \sum_{m=1}^\infty \mP{S_{m,j} < L}\mP{M=m}.
\end{align}
The distribution of the number of episodes in an interval, $M$, can be bounded by a geometric random variable, whose success probability is $\mP{R_j}$:
\begin{align}\label{eq:GeomEpisodes}
  \mP{M=m} \leq \paren{1-\mP{R_j}}^m\mP{R_j}.
\end{align}
This can be justified in the following manner: let us examine the $j^\text{th}$ interval. An episode beginning at $t_j$ has a simultaneous cession probability bounded by $\mP{R_j}$. Using the index $k$ to denote the different episodes in the interval, we denote by $t_{k-1}$ the beginning of the $k^\text{th}$ episode, where $t_0\triangleq t_j$. We also denote the probability of a simultaneous cession at the beginning of an episode by $P_k$. Since $\mP{R_j}$ monotonously decreases over time, $P_k < P_{k-1}, \;\forall k = 1..m$.
For a certain interval we have that:
\begin{align*}
  \mP{M=0} &= P_0 = \mP{R_j},\\
  \mP{M=1} &= \paren{1-P_0}P_1 \leq \paren{1-P_0}P_0 = \paren{1-\mP{R_j}}\mP{R_j},\\
  \mP{M=2} &= \paren{1-P_0}\paren{1-P_1}P_2 \leq \paren{1-P_0}\paren{1-P_1}P_1 \leq \paren{1-P_0}\paren{1-P_0}P_0 = \paren{1-\mP{R_j}}^2\mP{R_j},
\end{align*}
where the last inequality can be shown to hold for all $P_k < 0.5$ by differentiating the function $f\paren{p} = \paren{1-p}p$. Defining $T_4 = \min_t\set{t:P_k\paren{t} < 0.5}$, we have that \eqref{eq:GeomEpisodes} holds for all $t > T_4$.

Let us continue developing \eqref{eq:Regret_Int}. Applying Hoeffding's inequality for i.i.d random variables, the use of which we will justify shortly, we have that
\begin{align}\label{eq:Hoeffding}
  \mP{S_{m,j} < L}  = \mP{S_{m,j} - \mE{S_{m,j}} < L - \mE{S_{m,j}}} \leq e^{-\frac{2\alpha^2}{\sum_{i=1}^m\paren{b_i-a_i}^2}},
\end{align}
where $\alpha = \mE{S_{m,j}} - L$, $a_i = 0$ and $b_i \leq t_{j+1}^\beta$. In practice, the different $U_i$'s are not identically distributed. However, since they differ only in the value of $b_i$ and the expression in \eqref{eq:Hoeffding} is monotonously increasing in $b_i$, we can use the value at the end of the interval, $b_{j+1}$, for our bound. Now, $\sum_{i=1}^m\paren{b_i-a_i}^2 \leq m t_{j+1}^{2\beta}$, and, similarly, $\mE{S_{m,j}} \geq \frac{m}{2}t_{j}^\beta$.
Therefore, our bound is
\begin{align*}
  \mP{S_{m,j} < L} \leq
  e^{-\frac{2\paren{\frac{m}{2}t_{j}^\beta - t^{\beta/2}}^2}{m t_{j+1}^{2\beta}}},
\end{align*}
where we substituted $T=0$ for clarity. We will adopt this substitution, which only affects constants, for the rest of our derivation.

Developing the exponent further, we have
\begin{align*}
  \mP{S_{m,j} < L} &\leq
  \exp\paren{-\frac{2}{m t_{j+1}^{2\beta}}
  \paren{\frac{m^2}{4}t_j^{2\beta} - m t_j^\beta t^{\beta/2} + t^\beta}} \\
   &= \exp\paren{-\frac{m t_j^{2\beta}}{2t_{j+1}^{2\beta}}
   + \frac{2t_j^{\beta}t^{\beta/2}}{t_{j+1}^{2\beta}}
   - \frac{2t^\beta}{m t_{j+1}^{2\beta}}} \\
   &\leq \exp\paren{-\frac{m t_j^{2\beta}}{2t_{j+1}^{2\beta}}
   + \frac{2t_j^{\beta}t^{\beta/2}}{t_{j+1}^{2\beta}}}
\end{align*}

Returning to \eqref{eq:Regret_Int} and substituting \eqref{eq:GeomEpisodes} and \eqref{eq:Hoeffding}, we have
\begin{align}
\begin{split}\label{eq:combine_C_j}
  \mP{C_j}
  &\leq \sum_{m=1}^\infty \paren{1-\mP{R_j}}^m \mP{R_j}
  e^{-\frac{m t_j^{2\beta}}{2t_{j+1}^{2\beta}}
   + \frac{2t_j^{\beta}t^{\beta/2}}{t_{j+1}^{2\beta}}} \\
  &\leq \mP{R_j}\sum_{m=1}^\infty
  e^{-m \mP{R_j} -\frac{m t_j^{2\beta}}{2t_{j+1}^{2\beta}}
   + \frac{2t_j^{\beta}t^{\beta/2}}{t_{j+1}^{2\beta}}} \\
   &\leq \mP{R_j}e^{\frac{2t_j^{\beta}t^{\beta/2}}{t_{j+1}^{2\beta}}}
   \sum_{m=1}^\infty e^{-\frac{m}{8}} \\
   &<  8\mP{R_j}e^{\frac{2t_j^{\beta}t^{\beta/2}}{t_{j+1}^{2\beta}}},
\end{split}
\end{align}
where the second inequality is due to the fact that $\paren{1-xy}^n\leq 1 - x + e^{-ny}$ for $0 \leq x,y \leq 1,\; n>0$, as shown in \cite{CoverThomas1991}. In the third inequality we omit the term $-m \mP{R_j}$ from the exponent, as it does not provide a considerable contribution to the decay of the exponent for large values of $t$. Also, observing that $t_j = j t^{\beta/2}$ results in
\begin{align*}
  \frac{t_j^{2\beta}}{t_{j+1}^{2\beta}}
  = \paren{\frac{j t^{\beta/2}}{\paren{j+1} t^{\beta/2}}}^{2\beta}
  = \paren{\frac{j}{j+1}}^{2\beta}
  > \frac{1}{4},
\end{align*}
with the last statement being valid for $j\geq 1, \beta < 1$. This justifies the penultimate inequality in \eqref{eq:combine_C_j}.
All that remains now is to bound the exponent left in \eqref{eq:combine_C_j}, $\exp\paren{\frac{2t_j^{\beta}t^{\beta/2}}{t_{j+1}^{2\beta}}}$. In order to obtain a constant bound for this expression, independent of $t$, we need the power of $t$ in the exponent to be negative. Let us develop the expression and observe its behavior for $j\geq j_0 = \ceil{t^{\beta/2}}$:
\begin{align*}
  \exp\paren{\frac{2t_j^{\beta}t^{\beta/2}}{t_{j+1}^{2\beta}}}
  = \exp\paren{\frac{2\paren{j t^{\beta/2}}^\beta t^{\beta/2}}{\paren{j+1}^{2\beta}t^{\beta^2}}}
  = \exp\paren{\frac{2j^\beta}{\paren{j+1}^{2\beta}}\cdot t^{\frac{\beta-\beta^2}{2}}}
  \leq \exp\paren{\frac{2}{j^{\beta}}\cdot t^{\frac{\beta-\beta^2}{2}}}
  \leq \exp\paren{2 t^{\frac{\beta}{2}-\beta^2}},
\end{align*}
where the power of $t$ is negative for any $\beta>\frac{1}{2}$.
Combining this with \eqref{eq:combine_C_j} and \eqref{eq:bound_R_j}, we have that the probability of accumulating regret in the $j^{\text{th}}$ interval is bounded:
\begin{align}\label{eq:bound_C_j}
  \mP{C_j} \leq \frac{8C_2}{t^{\beta/2}}\quad\forall j > t^{\beta/2}, \beta > \frac{1}{2}.
\end{align}
We can now finally calculate the expected regret contributed by the availability mechanism:
\begin{align*}
  R^A\paren{t} &\leq t^{\beta/2}\sum_{j=0}^N \mP{C_j}
  \leq j_0t^{\beta/2} + t^{\beta/2}\sum_{j=j_0}^N \mP{C_j}
  \leq t^\beta + t^{\beta/2}\sum_{j=t^{\beta/2}}^{t^{1-\beta/2}}\frac{8C_2}{t^{\beta/2}}
  \leq t^\beta + 8C_2 t^{1-\beta/2},
\end{align*}
where we used the fact that $N = t^{1-\beta/2}$ and chose $j_0 = t^{\beta/2}$.

This expression can either be optimized by tuning $\beta$, or, for simplicity, bounded:
\begin{align}\label{eq:availability_bound}
  \mE{R^A\paren{t}} \leq C_3 t^\beta,
\end{align}
where $C_3 = 1+8C_2$, with the bound holding for $\beta > 2/3$.

\qed

\subsection{Proof of \lemref{lem:explore}}
At every time $t$, each user $n$ has an exploration probability $\e_t$, where
\begin{align*}
  \e_t = \min\paren{1, \frac{cK^2}{d^2\paren{K-1}t}},
\end{align*}
where $c > 0$ and $d>0$ are constants, and $K$ is the number of arms.

Let us denote $m = \ceil{\frac{cK^2}{d^2\paren{K-1}}}$. For a single user, for $t > m$, the expected regret accumulated up till time $t$ is bounded:
\begin{align*}
  R_k^E\paren{t} \leq \sum_{\tau = 1}^t \e_t = m + \sum_{\tau = m+1}^t \frac{cK^2}{d^2\paren{K-1}t}.
\end{align*}
Bounding a discrete sum by an integral we have that
\begin{align*}
  R_k^E\paren{t} &\leq m + \frac{cK^2}{d^2\paren{K-1}}\int_{m}^{t-1}\frac{1}{x}dx = M + \frac{cK^2}{d^2\paren{K-1}}\log\frac{t-1}{m}.
\end{align*}
The total regret for all $N$ users is therefore also bounded:
\begin{align*}
  R^E\paren{t} \leq Nm +  \frac{cK^2N}{d^2\paren{K-1}}\log\frac{t-1}{m} \leq Nm + \frac{cK^2N}{d^2{K-1}}\log t.
\end{align*}

\subsection{Proof of \propref{prop:dynamic1}}
Let us observe an arm $k\in K^*_{N}$. In the steady state, this arm is sampled by some user $n \leq N$, and it is marked as taken by all users who are sampling ``worse'' arms, i.e. arms with worse expected rewards. Had it not been marked as taken, users would be colliding when trying to access it.

If user $n$ becomes inactive at time $t$, the arm she had been sampling up till then will become available. However, this change will not be noticed by users sampling other arms immediately - they will only learn of this when their unavailability period for this arm expires. In the worst case, all users marked arm $k$ as taken just before user $n$ left, and the length of the unavailability period they drew was maximal - $t^\beta$. In this case, the arm will not be sampled intentionally for $t^\beta$ cycles, and may only be sampled occasionally due to the exploration factor $\e_t$, which is very small at this point in time.
However, after the unavailability period expires, users that were sampling ``worse'' arms will move on to sample arm $k$ with probability of at least $1-\e_t$, thus re-including it in the set of regularly sampled arms.

\subsection{Proof of \propref{prop:dynamic2}}
We begin by stating a lemma that we will need for the proof, taken from \cite{HLP1988}.
\begin{lem}%[HLP 41, 39]
\label{lem:sum2}
  For $x,y,\rho,p,q \geq 0$ such that $\frac{1}{p}+\frac{1}{q}=1$, the following holds:
  \begin{align*}
    \paren{x+y}^\rho \leq p^\rho x^\rho + q^\rho y^\rho.
  \end{align*}
\end{lem}

As mentioned in \secref{sec:dynamic_users}, the regret bound in \propref{prop:dynamic2} is obtained from a worst case analysis. The regret is then bounded by
\begin{align*}
  R\paren{t} \leq \sum_{i=1}^{N\paren{t}}\tau_i,
\end{align*}
where $\tau_0 = t$ and $\tau_i = \paren{\sum_{j=0}^{i-1}\tau_j}^\beta$.
Let us observe each member of the sum:
\begin{align*}
  \tau_i = \paren{\sum_{j=0}^{i-1}\tau_j}^\beta
  = \paren{\sum_{j=0}^{i-2}\tau_j + \tau_{i-1}}^\beta
  \leq 2^\beta\paren{\sum_{j=0}^{i-2}\tau_j}^\beta + 2^\beta\tau_{i-1}^\beta
  = 2^\beta\paren{\tau_{i-1} + \tau_{i-1}^\beta}
  \leq 2^{\beta+1}\tau_{i-1},
\end{align*}
where the first inequality is an application of \lemref{lem:sum2} and the second follows from the fact that $\forall a>0$, $a^\beta \leq a$, since $\beta<1$.
Repeating this process recursively, we obtain the following:
\begin{align*}
  \tau_i \leq \paren{2^{\beta+1}}^{i-1}t^\beta.
\end{align*}
As a result, we obtain the regret bound:
\begin{align*}
  R\paren{t} \leq \sum_{i=1}^{N\paren{t}}\tau_i \leq t^\beta\sum_{i = 1}^{N\paren{t}}\paren{2^{\beta+1}}^{i-1}
  = \frac{2^{\paren{\beta+1}\paren{N\paren{t}-1}}-1}{2^{\beta+1}-1}t^\beta.
\end{align*}

\end{document}